\documentclass{article}
\usepackage{iclr2027_conference,times}

\usepackage{hyperref}
\usepackage{url}
\usepackage{microtype}
\usepackage{graphicx}
\usepackage{amsmath,amssymb,amsfonts,amsthm,mathtools,bm}
\usepackage{booktabs}
\usepackage{multirow}
\usepackage{enumitem}
\usepackage{xcolor}
\usepackage{booktabs}
\usepackage{makecell}
\usepackage{siunitx}
\usepackage{changes}
\usepackage{comment}

\hypersetup{
    colorlinks,
    linkcolor={red!50!black},
    citecolor={green!50!black},
    urlcolor={blue!80!black}
}

\usepackage{amsmath,amsfonts,bm}

\def\eqref#1{equation~\ref{#1}}
\def\Eqref#1{Equation~\ref{#1}}

\def\1{\bm{1}}

\DeclareMathAlphabet{\mathsfit}{\encodingdefault}{\sfdefault}{m}{sl}
\SetMathAlphabet{\mathsfit}{bold}{\encodingdefault}{\sfdefault}{bx}{n}

\newtheorem{theorem}{Theorem}
\newtheorem{proposition}[theorem]{Proposition}

\theoremstyle{definition}
\newtheorem{definition}[theorem]{Definition}
\newtheorem{remark}[theorem]{Remark}

\newcommand{\btheta}{\bm{\theta}}

\newcommand{\scmap}{f_t(\btheta)}
\newcommand{\llogit}{\ell^{\mathrm{logit}}}
\newcommand{\lnce}{\ell^{\mathrm{InfoNCE}}}
\newcommand{\ldist}{\ell^{\mathrm{dist}}}
\newcommand{\lmrg}{\ell^{\mathrm{mrg}}}

\makeatletter
\@ifundefined{proposition}{\newtheorem{proposition}{Proposition}}{}
\@ifundefined{theorem}{\newtheorem{theorem}{Theorem}}{}
\@ifundefined{definition}{}{}
\@ifundefined{remark}{}{}
\makeatother

\definechangesauthor[name={FJ}, color=blue]{FJ}

\graphicspath{{../docs/figures/}{figures/}}

\title{Contrastive Linearized Laplace:
Loss Geometry, Valid Repairs, and Reliable Cores}

\author{%
  F. J. Sáez-Maldonado \\
  University of Granada \\
  {\footnotesize\texttt{fjaviersaezm@ugr.es}}
  \And
  Luis A. Ortega \\
  Aalborg University \\
  {\footnotesize\texttt{laoa@cs.aau.dk}}
  \And
  Andrés R. Masegosa \\
  Aalborg University \\
  {\footnotesize\texttt{arma@cs.aau.dk}}
  \And
  Thomas D. Nielsen \\
  Aalborg University \\
  {\footnotesize\texttt{tdn@cs.aau.dk}}
}
\iclrfinalcopy

\begin{document}
\maketitle

\begin{abstract}
While contrastive learning is widely used to learn deterministic representations, it does not inherently provide estimates of uncertainty. Post-hoc uncertainty estimation methods, such as the Linearized Laplace approximation (LLA), offer a way to quantify uncertainty without retraining the encoder. However, we show that directly using the LLA induced distribution over the contrastive embedding does not yield reliable pointwise uncertainty estimates for downstream tasks: embedding variance alone does not reflect whether decisions such as nearest-neighbor retrieval or cluster assignment are stable. To address this limitation, we develop a post-hoc framework based on LLA that propagates encoder uncertainty through the downstream task. We characterize the curvature of contrastive objectives and introduce positive-semidefinite repairs where needed to obtain valid local Gaussian approximations without retraining the encoder. We test the proposed approach in metagenomic binning and image classification tasks. For metagenomic binning, posterior coassignment probabilities support consensus partitions and the selection of reliable cores, leading to improvements over existing baselines. Additional image classification experiments show calibration improvements on standard benchmarks and real histopathology datasets. 
\end{abstract}

\section{Introduction}

Contrastive learning is the standard way to train representations for retrieval, verification, clustering, metric learning, and self-supervised pretraining \citep{Hadsell2006Dimensionality,chen2020simple,Khosla2020SupCon,oord2018representation}. Unlike ordinary supervised prediction, these methods are trained via learning \emph{relational} decisions: whether two examples match, which candidate is closest, or which samples belong to the same cluster. In applications such as image retrieval, genome binning, and pathology review, it is therefore important not only to obtain accurate representations, but also to identify which of these relations can be trusted. Standard contrastive encoders provide deterministic embeddings and scores,
but no direct measure of how robust the resulting relational decisions are.

Bayesian inference provides a natural framework for this problem: uncertainty in the shared encoder induces joint uncertainty over all embeddings and their relations. Among approximate methods, the Linearized Laplace Approximation (LLA) is attractive because it constructs a tractable Gaussian posterior around an already trained network \citep{Daxberger2021LaplaceRedux,Deng2022AcceleratedLLA}. However, applying LLA to contrastive learning raises two difficulties. First, contrastive losses act on relational outputs (such as pair scores, embedding differences, or candidate logits) rather than on independent example-wise predictions. Hence, the resulting precision need not be positive semidefinite. Thus, a linearized encoder can yield an invalid Gaussian posterior. Second, even with a valid posterior, uncertainty in an individual embedding or pair score does not necessarily reflect uncertainty in the downstream decision, which may depend jointly on several competing relations and their posterior correlations. The key distinction is between uncertainty in a representation and uncertainty in a decision. An embedding can move substantially in the representation space (high \textit{representation uncertainty}) while remaining grouped with the same neighbors (low \textit{decision uncertainty}). Conversely, a small perturbation can change its assignment when it lies close to a boundary. 

We address the first problem by showing that posterior validity is determined by the geometry of the contrastive loss: logit-Bernoulli and InfoNCE objectives have positive-semidefinite (PSD)  output curvature, whereas distance-Bernoulli and classical margin losses can exhibit negative  curvature for negative pairs. For the latter, we introduce PSD curvature constructions that make the resulting local posterior well defined. This provides a unified treatment of pairwise and candidate-set contrastive objectives and extends previous observations for Bayesian metric learning \citep{warburg2023bayesian}.

We address the second problem by distinguishing \emph{representation uncertainty} from \emph{decision uncertainty}. Pointwise embedding variance is a valid posterior statistic, but it ignores decision-relevant directions and competing relations. Pairwise score variance is closer to retrieval and verification, but remains a marginal quantity. Decisions such as nearest-neighbor retrieval, graph construction, and clustering depend jointly on several scores, their margins, and their covariance. We therefore propagate representation uncertainty through the downstream rule and measure the stability of the decision uncertainty. This yields \emph{reliable cores}: subsets of a downstream structure whose assignments remain stable under posterior perturbations, summarized through quality--coverage tradeoffs.

Our contributions are as follows. First, we formulate a relational LLA for pairwise and candidate-set contrastive losses and characterize posterior validity through their output-space geometry. Second, we provide PSD constructions for distance-based objectives with indefinite curvature. Third, we formalize why pointwise embedding variance and marginal pair-score variance are insufficient for general relational decisions. Finally, we propose posterior-sampled reliable cores that quantify uncertainty after applying the downstream rule rather than through a surrogate variance threshold. We demonstrate state-of-the-art performance on DNABERT-S-style genome clustering \citep{Zhou2024DNABERTS} benchmarks. We further perform experiments in image classification on EMNIST \citep{Cohen2017EMNIST}, CIFAR-100 \citep{krizhevsky2009learning} and SICAPv2 pathology dataset \citep{SilvaRodriguez2020SICAPv2}.

\section{Related Work}
\label{sec:related_work}

Uncertainty in metric and contrastive learning is often modeled through
distribution-valued embeddings, including instance and face embeddings,
triplet learning, cross-modal retrieval, video representations, and
InfoNCE-based objectives
\citep{Oh2019HedgedEmbedding,Shi2019ProbabilisticFaceEmbeddings,
Warburg2021BayesianTriplet,Chun2021ProbabilisticCrossModal,
Park2022ProbabilisticVideo,Kirchhof2023ProbabilisticContrastive}.
Other work estimates representation reliability or uses uncertainty to
select and reweight contrastive pairs during training
\citep{Park2024RepresentationReliability,Wang2023UCoL,
Devalraju2025UncertaintyGuidedMetric,Chang2025UncertaintyAwareClustering}.
These methods primarily attach uncertainty to individual representations
or use it as a training signal. We instead infer an epistemic posterior
over the shared encoder, inducing joint uncertainty over its relational
outputs.

Most closely related, \citet{warburg2023bayesian} applies
Laplace inference to pairwise metric learning and addresses invalid local
curvature, while \citet{Cheng2025DeepProbabilisticBinaryEmbedding} uses
weight-space Laplace inference for cross-modal retrieval. We generalize
this perspective to a relational score map covering both pairwise and
candidate-set objectives and characterize posterior validity through
output-space curvature. At the downstream level, our approach is related
to posterior-sampled retrieval uncertainty
\citep{Gomez2025OverTopOne}, risk-controlled retrieval
\citep{Cai2025RiskControlledRetrieval}, selective prediction
\citep{ElYaniv2010Selective,Geifman2017SelectiveClassification}, and
selective clustering \citep{Marandon2025FalseSelection}. Rather than
thresholding an isolated representation or score, we propagate posterior
encoder samples through the downstream relational rule and measure the
stability of the resulting structure.

General Laplace and LLA methods are reviewed in
Appendix~\ref{app:related_work}.

\section{Laplace Approximations and Contrastive Objectives}
\label{sec:background}

Contrastive learning learns representations by comparing examples rather
than predicting an individual target for each input. Given an encoder
$h_{\btheta}:\mathcal{X}\rightarrow\mathbb{R}^d$, an input
$\mathbf{x}_i$ is mapped to an embedding
$\mathbf{z}_i=h_{\btheta}(\mathbf{x}_i)$. Training encourages
embeddings of \emph{positive} examples to be similar and those of
\emph{negative} examples to be dissimilar, where positive and negative
relations may arise from labels, augmentations, or other notions of
semantic similarity.

We denote by $t$ a contrastive training unit containing the examples and
relations required to compute one loss term. Associated with $t$ is a
score map $\scmap\in\mathbb{R}^{K_t}$ and a loss
$\ell_t:\mathbb{R}^{K_t}\rightarrow\mathbb{R}$, giving the training
objective $\sum_t\ell_t(\scmap)$. For example, in
InfoNCE~\citep{oord2018representation}, $t$ contains an anchor $i$, a
candidate set $C_i$, and a positive candidate $p(i)$. Its scores are the
anchor--candidate logits
$f_{t,j}=s(\mathbf{z}_i,\mathbf{z}_j)/\tau$, $j\in C_i$, and
\begin{equation}
    \lnce_t(f_t)
    =
    -f_{t,p(i)}
    +
    \log\sum_{j\in C_i}\exp(f_{t,j}),
\end{equation}
where $s$ is a similarity function and $\tau>0$ is the temperature.

\paragraph{Contrastive Linearized Laplace Approximation}

Consider the score map
$\scmap$ with loss
$\ell_t(\scmap)$. Under a Gaussian parameter prior
$\mathcal{N}(\btheta_0,\boldsymbol{\Lambda}_0^{-1})$, the MAP
estimate $\btheta^\star$ minimizes
\begin{equation}
    \mathcal{L}(\btheta)
    =
    \sum_t \ell_t(\scmap)
    +
    \frac{1}{2}
    (\btheta-\btheta_0)^\top
    \boldsymbol{\Lambda}_0
    (\btheta-\btheta_0).
\end{equation}
The Laplace approximation constructs
$q(\btheta)
=\mathcal{N}(\btheta^\star,\boldsymbol{\Sigma})$
using the local curvature of this objective. For each loss term,
\begin{equation}
    \nabla_{\btheta}^2\ell_t
    =
    \mathbf{J}_t^\top\mathbf{H}_t\mathbf{J}_t
    +
    \sum_{k=1}^{K_t}
    [\nabla_{f_t}\ell_t]_k
    \nabla_{\btheta}^2 f_{t,k},
    \label{eq:hessian_decomposition}
\end{equation}
where
$\mathbf{J}_t=\partial f_t/\partial\btheta$ and
$\mathbf{H}_t=\nabla_{f_t}^2\ell_t$, evaluated at
$\btheta^\star$. The Linearized Laplace Approximation (LLA) replaces the model locally by
$\scmap
\approx f_t(\btheta^\star)
+\mathbf{J}_t(\btheta-\btheta^\star)$.
The second term in~\Eqref{eq:hessian_decomposition} therefore vanishes,
yielding the Generalized Gauss--Newton (GGN) precision
\begin{equation}
    \boldsymbol{\Sigma}^{-1}
    =
    \boldsymbol{\Lambda}_0
    +
    \sum_t
    \mathbf{J}_t^\top\mathbf{H}_t\mathbf{J}_t.
    \label{eq:lla_gnn}
\end{equation}

\paragraph{Validity of the defined precision} 
When $\mathbf{H}_t\succeq0$, each data contribution is positive
semidefinite and~\Eqref{eq:lla_gnn} defines a valid Gaussian precision
given a positive definite prior or damping. Thus, posterior validity depends directly on the curvature of the
contrastive loss with respect to its relational scores. If $\boldsymbol{\Lambda}_0\succeq0$ and $\mathbf{H}_t\succeq0$ for every
contrastive unit $t$, then $\boldsymbol{\Sigma}^{-1}\succeq0$. If,
additionally, $\boldsymbol{\Lambda}_0\succ0$ or positive damping is
added, then $\boldsymbol{\Sigma}^{-1}\succ0$ and defines a proper local
Gaussian posterior~\citep{Daxberger2021LaplaceRedux}.

For InfoNCE, choosing $\scmap$ as the vector of anchor--candidate logits
satisfies this condition. This can be proved by simply writing
$\mathbf{p}_t=\operatorname{softmax}(\scmap)$,
\begin{equation}
    \mathbf{H}_t
    =
    \operatorname{diag}(\mathbf{p}_t)
    -
    \mathbf{p}_t\mathbf{p}_t^\top
    \succeq0,
\end{equation}
which is the covariance matrix of a categorical distribution over the
candidate set. Hence InfoNCE yields a PSD GGN contribution while
retaining the local coupling between candidates.

Positive semidefiniteness is not guaranteed for other contrastive
objectives. Consider the pairwise margin loss \citep{Hadsell2006Dimensionality}
\begin{equation}
\label{eq:margin:loss}
    \scmap=\mathbf{z}_i-\mathbf{z}_j,
    \qquad
    \lmrg_t(\scmap,y_t)
    =
    y_t\|\scmap\|_2^2
    +(1-y_t)[m-\|\scmap\|_2]_+^2,
\end{equation}
where $y_t=1$ denotes a positive pair and $y_t=0$ a negative pair. For a negative pair within the margin, $\mathbf{H}_t$ is indefinite due to negative tangential curvature; see Appendix~\ref{sec:geometry} for further details. More generally, distance- and margin-based
contrastive losses need not be convex in their relational scores. For these cases, we consider PSD repairs that either remove negative
local curvature (spectral repair) or retain the informative radial geometry of
distance-based losses (radial repair); their precise construction is given in
Appendix~\ref{sec:geometry}. We additionally study the Empirical Fisher
as a PSD alternative in Appendix~\ref{app:ef}. Unless stated otherwise,
our experiments use the \emph{spectral} repair of the GGN approximation, with all alternatives
compared in Appendix~\ref{app:ablations}.

\section{Marginal Variance Is Not Downstream Uncertainty}
\label{sec:variance}

\begin{figure}
    \centering
    \includegraphics[width=0.98\linewidth]{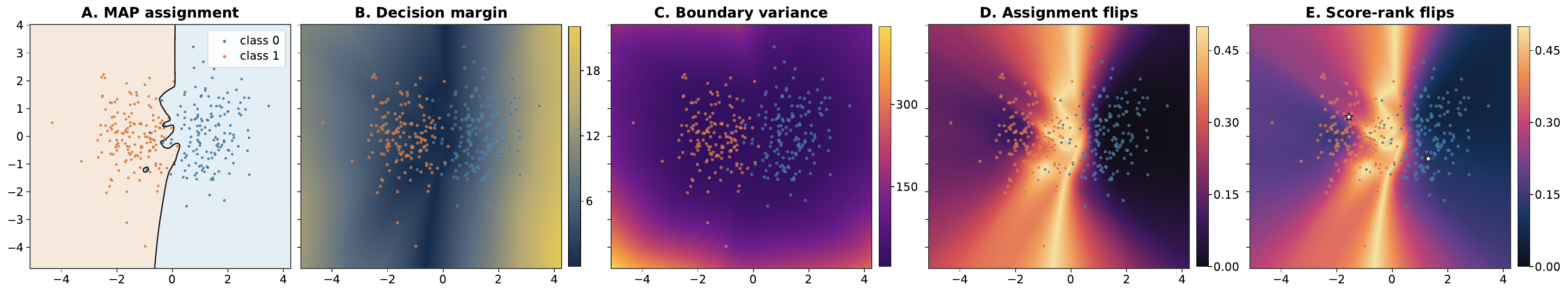}
    \caption{Synthetic two-blob example showing that posterior variance is not itself downstream uncertainty. Assignment uncertainty depends on the decision margin and variance along the boundary normal, while score-rank uncertainty depends on the posterior variance of a score difference.}
    \label{fig:synthetic:toy}
\end{figure}

A repaired LLA defines a Gaussian posterior over the encoder parameters, $q(\btheta) = \mathcal{N}(\btheta^\star,\boldsymbol{\Sigma}_\theta).
$
Contrastive LLA defines this posterior by locally linearizing the relational score map with respect to the encoder parameters. The resulting parameter posterior can in turn be propagated to the representation space to characterize uncertainty in individual embeddings. Specifically, for an input $\mathbf{x}$, we linearize the embedding network around the MAP parameters,
\[
h_{\btheta}(\mathbf{x})
\approx
h_{\btheta^\star}(\mathbf{x})
+
\mathbf{J}_{\mathbf{x}}
(\btheta-\btheta^\star),
\]
which induces a Gaussian distribution in embedding space,
\[
h_{\btheta}(\mathbf{x})\mid\mathcal D
\sim
\mathcal N(\boldsymbol{\mu}_{\mathbf{x}},\boldsymbol{\Sigma}_{\mathbf{x}}),
\qquad
\boldsymbol{\mu}_{\mathbf{x}}=h_{\btheta^\star}(\mathbf{x}),
\qquad
\boldsymbol{\Sigma}_{\mathbf{x}}
=
\mathbf{J}_{\mathbf{x}}\boldsymbol{\Sigma}_\theta\mathbf{J}_{\mathbf{x}}^\top.
\]
We refer to these induced distributions as Gaussian embeddings.
These moments are useful, but they are not automatically uncertainties for a downstream task.
A downstream task does not usually ask how much an embedding or a score varies in isolation. It
asks whether a decision changes: whether an item changes class, crosses a cluster boundary, enters
a graph, or changes its rank relative to another item. The relevant uncertainty is therefore tied to
the margin or decision rule used by the task.

In the following theoretical results, we use a simple linear binary assignment problem to illustrate why Gaussian embedding uncertainty may not directly reflect uncertainty in a downstream task. This covers a linear classifier, and it also describes the boundary between two Euclidean nearest-centroid regions, where the sign of an affine function determines which
centroid is closer. Figure~\ref{fig:synthetic:toy} serves as illustrative guidance on these following results.

\begin{proposition}
\label{prop:variance-not-assignment-uncertainty}
Let \(h_{\bm{\theta}}\) be a (linearized) contrastive encoder and \(q(\bm{\theta})\) be a (repaired) LLA posterior, then, for any \(\mathbf{x} \in \mathcal{X}\), 
\(
    h_{\bm{\theta}}(\mathbf{x})\mid\mathcal{D}
    \sim
    \mathcal{N}(\boldsymbol{\mu},\boldsymbol{\Sigma})
\). 
Consider a linear binary decision boundary defined as
\(
    D(\mathbf{z})
    :=
    \mathbf{1}\{\mathbf{a}^{\top}\mathbf{z}+b\ge 0\}
\).
Then, denoting the standard gaussian c.d.f. as \(\Phi\), the posterior probability of changing the assignment made at the posterior mean is
\[
    \mathbb{P}\!\left[
        D(h(\mathbf{x}))\neq D(\boldsymbol{\mu})
        \mid\mathcal{D}
    \right]
    =
    \Phi\!\left(
        -
        \frac{
            |\mathbf{a}^{\top}\boldsymbol{\mu}+b|
        }{
            \sqrt{\mathbf{a}^{\top}\boldsymbol{\Sigma}\mathbf{a}}
        }
    \right)\,.
\]
\end{proposition}

This result provides some intuition about the relation between Gaussian embeddings and downstream uncertainty. First, two Gaussian embeddings with the same pointwise covariance can have different assignment uncertainty if their means have different margins. Conversely, two Bayesian embeddings whose covariance matrices have the same trace can have different assignment uncertainty if their posterior variance is oriented differently relative to the boundary. The pointwise covariance is useful only after the downstream assignment has specified which direction and which margin matter. We explain this visually in Figure~\ref{fig:assignment_uncertainty} in Appendix~\ref{sec:appendix:marginal:uncertainty:proofs}.

Second, the closed-form calculation also indicates another limitation of using a single pointwise variance score
in more general downstream tasks: in realistic scenarios, the effective boundary may be implicit, discrete,
or defined jointly by many examples. There may be no single fixed hyperplane \((\mathbf{a},b)\) for
which the above formula is available. In such cases, the uncertainty object must be defined at the
level of the downstream decision itself rather than by a universal scalar summary of
\(\boldsymbol{\Sigma}\).

A similar issue arises in pairwise scores. Consider the case of a retrieval task, denoting
\(\mathbf{x}_i\) to a query and \(\mathbf{x}_j\) and \(\mathbf{x}_k\) two
candidate items, with the system selecting the candidate with the larger
similarity score. Generally, the downstream decision depends on the comparison margin between their scores.

\begin{proposition}
\label{prop:pairwise-variance-downstream}
Suppose a (repaired) LLA posterior \(q(\bm{\theta}) = \mathcal{N}(\bm{\theta}^\star, \bm{\Sigma})\), the two posterior (linearized) scores are jointly Gaussian under
\[
\begin{pmatrix}
    s_{\btheta}(\mathbf{x}_i,\mathbf{x}_j)\\
    s_{\btheta}(\mathbf{x}_i,\mathbf{x}_k)
\end{pmatrix} 
\sim 
\mathcal{N}\left(
    \begin{pmatrix}
        \mu_{ij}\\
        \mu_{ik}
    \end{pmatrix},
    \begin{pmatrix}
        V_{ij}     & C_{ij,ik}\\
        C_{ik, ij} & V_{ik}
    \end{pmatrix}
 \right).
\]
For the downstream comparison
\(
    A_i
    =
    \mathbf{1}
    \left\{
        s_{\btheta}(\mathbf{x}_i,\mathbf{x}_j)
        \ge
        s_{\btheta}(\mathbf{x}_i,\mathbf{x}_k)
    \right\}
\),
 the probability of changing the score rank w.r.t the MAP is given by
\[
    \mathbb{P}
    \left[
        A_i(\btheta)
        \neq
        A_i(\btheta^{\star})
    \right]
    =
    \Phi\!\left(
        -
        \frac{
            |\mu_{ij}-\mu_{ik}|
        }{
            \sqrt{V_{ij}+V_{ik}-2C_{ij,ik}}
        }
    \right),
\]
\end{proposition}

The score-level result has the same interpretation as the embedding-level result. The marginal
variances \(V_{ij}\) and \(V_{ik}\) describe the uncertainty of two scores separately, but the
comparison is determined by their difference. Hence the covariance \(C_{ij,ik}\) is part of the
downstream uncertainty: two uncertain scores may move together and leave their difference stable,
or move against each other and make the comparison unstable. The correct variance for the ranking
decision is the variance of the comparison margin, not the marginal variance of either
score alone. See Figure~\ref{fig:score_rank_uncertainty} in Appendix~\ref{sec:appendix:marginal:uncertainty:proofs}.

These two cases show the same principle at two levels. Pointwise embedding covariance becomes
task-relevant only after it is applied to an assignment margin. Pairwise score variance becomes
task-relevant only after it is applied to a score comparison or threshold margin. When these margins
are explicit, the resulting stability probabilities can be written in closed form. When the downstream
rule is more complex and its boundary is not available analytically, uncertainty must instead be
defined by the induced posterior over the downstream decisions themselves.

\section{Posterior Expected Downstream Tasks and Reliable Cores}
\label{sec:decision-posteriors}

The previous section shows that marginal embedding variances and marginal score variances are not, by themselves, downstream uncertainties. What we want to measure is the downstream task predictive uncertainty after encoder uncertainty has been propagated through the downstream rule. In other words, we want an expectation under the (repaired) LLA posterior over encoders. We now describe our proposal for propagating the encoder uncertainty to the downstream task.

\paragraph{Classification}
For classification, propagating encoder uncertainty means accounting for how different posterior encoder realizations affect the predicted class probabilities. Each draw \(\btheta\sim q(\btheta)\) produces an embedding of \(\mathbf{x}_i\), which a  downstream classifier maps to a probability vector over the classes. We denote the probability assigned to class \(y\) under this realization by \(p_{\btheta}(y|\mathbf{x}_i)\).

Although these probabilities may vary across posterior samples, their coordinates always refer to the same classes. We can therefore combine the predictions by averaging the probability assigned to each class separately. This gives the posterior-averaged classifier
\begin{equation}
\label{eq:posterior_model_averaging_classification}
\bar p_i(y)
=
\mathbb{E}_{\btheta\sim q(\btheta)}
\left[
p_{\btheta}(y\mid\mathbf{x}_i)
\right] \approx
    \tfrac{1}{S}\textstyle \sum_{s=1}^{S}
    p_{\btheta^{(s)}}(y\mid\mathbf{x}_i) \quad,\quad \btheta^{(s)}\sim q(\btheta).
\end{equation}
Thus, \(\bar p_i(y)\) summarizes support for class \(y\) across the encoder realizations represented by the posterior. If these realizations confidently agree on a class, the averaged prediction remains concentrated on that class; disagreement distributes probability across competing classes. When the expectation is not available analytically, it can be approximated using posterior samples.

Thus, for probabilistic classification, the downstream expectation is Bayesian model averaging of the class probabilities under the LLA posterior over encoders. Typically, that expectation cannot be computed in closed form; however, in our experiments, a logistic regressor classifier is used over the embeddings, which induces Gaussian logits. The resulting class probabilities must still be averaged over the posterior; we approximate this expectation using the Monte Carlo estimator in Equation~\ref{eq:posterior_model_averaging_classification}.

\paragraph{Clustering} The same idea cannot be applied directly to the cluster labels. Cluster labels are
arbitrary: cluster \(1\) in one posterior sample does not necessarily correspond to cluster \(1\) in
another posterior sample. Therefore, averaging cluster labels, or averaging one-hot encodings of
cluster labels, is not meaningful without an additional label-alignment step.

Instead, we average a label-invariant clustering quantity: whether two points are assigned to the same
cluster. For each posterior sample \(\btheta^{(s)}\), we embed the evaluation set and apply the same clustering rule used at test time. This gives
sampled assignments \(\hat y_i^{(s)},\ i=1,\ldots,N, \ s=1,\ldots,S\).  Although the numerical labels \(\hat y_i^{(s)}\) are not comparable across samples, the event
\(\hat y_i^{(s)}=\hat y_j^{(s)}\) is comparable. We therefore build the posterior similarity matrix (PSM) \citep{bayesianclustering}, defined as
\begin{equation}
\label{eq:posterior_expected_coassignment}
    P_{ij}
    =
    \mathbb{P}_{\btheta\sim q(\btheta)}
    \left[
        i \text{ and } j \text{ are assigned to the same cluster}
    \right]
    \approx
    \frac{1}{S}
    \sum_{s=1}^{S}
    \mathbf{1}\{\hat y_i^{(s)}=\hat y_j^{(s)}\}.
\end{equation}
In other words, \(P_{ij}\) is the posterior probability that points \(i\) and \(j\) remain together under encoder perturbations, and the clustering analogue of the posterior-averaged class-probability vector in classification. Pairwise coassignment probabilities need not define a globally consistent partition. However, a final clustering step (e.g. spectral clustering) can be easily approached using \(1 - P\) as distance function, producing final consensus labels.

\subsection{Measuring reliability through selective retention}
\label{sec:selective:retention}

The posterior-averaged downstream output also provides a reliability score for selective retention.
The goal is to test whether the downstream result improves after rejecting samples whose posterior
behaviour is unstable. This gives a quality--coverage tradeoff: coverage is the fraction of retained
inputs, and quality is measured by the downstream metric on the retained set.

For clustering, reliability is computed from the PSM in \eqref{eq:posterior_expected_coassignment}, since raw cluster labels
cannot be averaged. Let \(c_i\) be the consensus label obtained from the final clustering step of item \(i\). We define
\begin{equation}
\label{eq:local_reliability}
    \rho_i
    =
    \tfrac{1}{N-1}
    \textstyle\sum_{j \neq i}
    P_{ij}\,
    \mathbf{1}\{c_i=c_j\}.
\end{equation}
as a measure of the robustness of the consensus assignment for item \(i\). This is high when the same-cluster relations around \(i\) are repeatedly recovered across
posterior samples. A \textit{reliable core} is obtained by retaining the highest-\(\rho_i\) items per cluster. As a result, uncertain or border items are discarded first, while items whose cluster membership is consistently supported by the posterior are retained. Varying the retained fraction gives a quality--coverage frontier, showing the downstream result improvement when posterior-unstable inputs are discarded. This is especially useful in clustering, where unstable
boundary points can obscure the reliable core structure of the partition.

\paragraph{Scalability} The described approach requires computing the probability matrix of the whole inputs, which requires \(\mathcal{O}(N^2)\) space. This is clearly prohibitive in real-world clustering scenarios. To alleviate this scalability issue, we approximate those quantities by a \(k\)-NN approach, where only the nearest \(k\) inputs (in embedding space) are used to compute these quantities.

\section{Experiments}

The experiments evaluate whether the geometric construction above yields useful uncertainty estimates after the contrastive posterior is propagated through downstream decisions. Appendix~\ref{app:experiments} includes further details in datasets, posterior, baselines, and metrics.

The experiments follow a transfer-style, post-hoc uncertainty-estimation protocol. Rather than retraining the encoder under a Bayesian objective, we take an already trained deterministic checkpoint as the MAP solution, fit a last-layer LLA using the  contrastive objective, and propagate the resulting posterior through the downstream task, this is closely related to post-hoc Laplace inference for contrastive metric learning~\citep{warburg2023bayesian}. This follows the standard use of Laplace and LLA to equip pretrained networks with uncertainty without modifying their original training procedure, often restricting inference to the last layer or a subnetwork for scalability~\citep{Daxberger2021LaplaceRedux,daxberger2021subnetwork,Deng2022AcceleratedLLA,Ortega2024VariationalLLA}. When the checkpoint was obtained using a different pretraining objective, the resulting posterior should therefore be interpreted as a task-adapted local posterior around the transferred representation, rather than as a posterior over the original process.

\begin{table}[t]
  \centering
  \caption{Curvature diagnostics and downstream posterior-consensus gains (w.r.t. MAP) on EMNIST and CIFAR100 using a pretrained CLIP encoder, averaged over 5 seeds. Output-Hessian columns report the minimum eigenvalue before repair ($\lambda_{\min}$), the proportion of negative eigenvalues before repair, and the mean per-output-Hessian Frobenius change induced by the repair ($\overline{\|\Delta H\|}_F$).}
  \label{tab:ablation:repairs}
  \resizebox{\textwidth}{!}{%
  \begin{tabular}{lllrrrrrr}
    \toprule
    & & & \multicolumn{3}{c}{Output-Hessian repair} & \multicolumn{3}{c}{Posterior change} \\
    \cmidrule(lr){4-6} \cmidrule(lr){7-9}
    Dataset & Loss & Repair & $\lambda_{\min}$ & Neg. rate & $\lVert\Delta H\rVert_F$ & $\Delta$ACC $\uparrow$ & $\Delta$ECE $\downarrow$ & $\Delta$NLL $\downarrow$\\
    \midrule
    EMNIST & Bernoulli distance & Radial & $-6.22\times 10^{-3}$ & $50.2\%$ & $22.48$ & $+0.009$ & $-0.030$ & $-0.108$ \\
     & Bernoulli distance & Spectral & $-6.22\times 10^{-3}$ & $50.2\%$ & $3.68\times 10^{-4}$ & $+0.008$ & $-0.030$ & $-0.108$ \\
     & Margin & Radial & $-1.22\times 10^{0}$ & $13.3\%$ & $23.20$ & $+0.009$ & $-0.030$ & $-0.109$ \\
     & Margin & Spectral & $-1.22\times 10^{0}$ & $13.3\%$ & $0.72$ & $+0.009$ & $-0.030$ & $-0.108$ \\
    \midrule
    CIFAR-100 & Bernoulli distance & Radial & $-3.32\times 10^{-11}$ & $49.3\%$ & $22.60$ & $+0.021$ & $-0.080$ & $-0.451$ \\
     & Bernoulli distance & Spectral & $-3.32\times 10^{-11}$ & $49.3\%$ & $9.13\times 10^{-13}$ & $+0.021$ & $-0.080$ & $-0.452$ \\
     & Margin & Radial & $-9.92\times 10^{-1}$ & $14.0\%$ & $23.24$ & $+0.022$ & $-0.080$ & $-0.450$ \\
     & Margin & Spectral & $-9.92\times 10^{-1}$ & $14.0\%$ & $0.64$ & $+0.022$ & $-0.080$ & $-0.449$ \\
    \bottomrule
  \end{tabular}
  }
\end{table}

\paragraph{The effect of the curvature repairs}

In Table~\ref{tab:ablation:repairs}, we show the impact of the different repair modes proposed in Appendix~\ref{sec:geometry}. For Bernoulli and Margin losses when applied to the CLIP encoder on both EMNIST and CIFAR100 in the classification task. The table shows the obtained proportion of negative eigenvalues, their minimum values before repair, and the change in matrix norm after repair. Spectral repair achieves substantially smaller curvature corrections than radial repair while preserving comparable downstream results. This advantage holds for both Bernoulli-distance and margin losses. See the implementation details of this experiment in Appendix~\ref{app:training-posterior}.

\subsection{Metagenomic Binning}
\label{sec:experiment:metagenomic}

The process of organizing DNA fragments from a mixed microbial sample to recover the full (source) genomes is known as metagenomic binning \citep{kuninBioinformaticiansGuideMetagenomics2008}. SemiBin2 ~\citep{pan2023semibin2} performs this task by clustering a contrastively learned embedding space in which contigs expected to originate from the same genome are brought closer together, while contigs expected to belong to different genomes are pushed apart. This is achieved through a pairwise margin loss, shown in Equation~\ref{eq:margin:loss}, making SemiBin2 a natural setting for evaluating our relational formulation. 
We use this contrastive structure to study whether the proposed approach can provide meaningful uncertainty over the learned pairwise relationships and, ultimately, over the resulting metagenomic representation.

We do not train a new contrastive binning model from scratch. Instead, we use the official pretrained SemiBin2 encoder as the MAP reference model and fit a repaired (using spectral repair) last-layer KFAC LLA posterior around this checkpoint.
For each posterior draw, we re-embed the contigs and apply $k$-means with $k=500$. We estimate the posterior coassignment probability $P_{ij}$ by the fraction of sampled partitions in which contigs $i$ and $j$ are assigned to the same cluster. A consensus partition is then obtained by applying spectral clustering to the $15$-nearest-neighbor graph constructed from $P$.\footnote{Appendix~\ref{app:ablations} shows the obtained results for different values of \(k\).} 
We evaluate on the Marine 5, Marine 6, Plant 5, and Plant 6 benchmark splits introduced with DNABERT-S~\citep{Zhou2024DNABERTS}, where the quality of a cluster is measured in terms of its F1 score, naming high quality (HQ) clusters to those with \(F1 > 0.9\).

\begin{figure}[t]
    \centering
    \includegraphics[width=\linewidth]{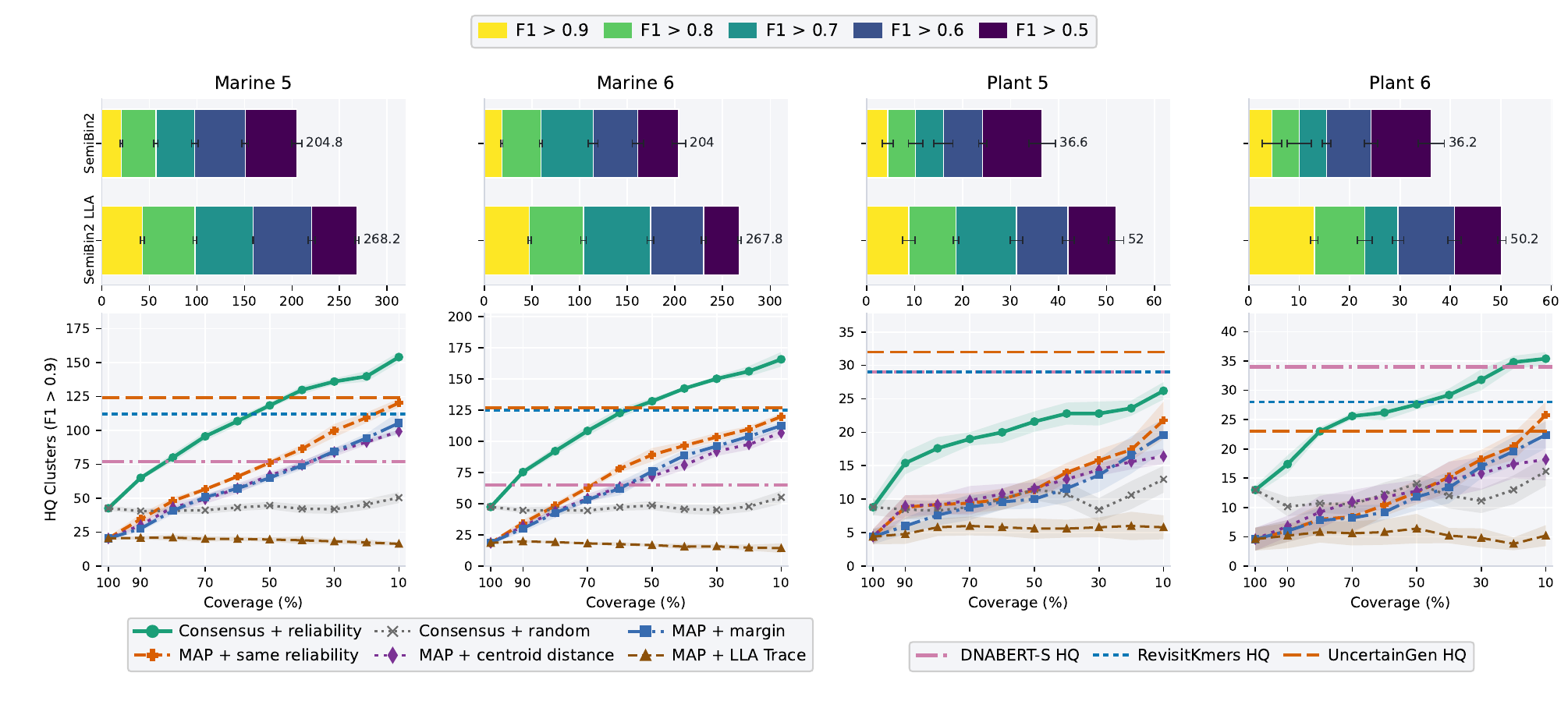}
    \caption{Genome reliable-core results on Marine and Plant splits. Top panels compare high-quality cluster counts for SemiBin2 and the LLA posterior-consensus variant across F1 thresholds. Bottom panels show high-quality clusters retained as coverage is reduced with different approaches.}
    \label{fig:genome-reliable-cores}
\end{figure}

Figure~\ref{fig:genome-reliable-cores} (top)
compares the SemiBin2 MAP clustering with the LLA posterior-consensus clustering across cluster-level $F_1$ thresholds. The figure shows that posterior LLA consensus (SemiBin2-LLA) increases the number of HQ clusters on all four datasets compared to the default MAP model (SemiBin2). The gains are not restricted to a single F1 threshold: the LLA variant shifts the full high-quality-cluster profile upward relative to the MAP baseline. This indicates that posterior sampling is not merely filtering isolated ambiguous contigs, but stabilizing the recovered bin structure.

Figure~\ref{fig:genome-reliable-cores} (bottom) evaluates selective clustering. The coverage curves show the number of HQ clusters obtained when progressively retaining only a fraction of the contigs (see Section~\ref{sec:selective:retention}). Here, \emph{coverage} is the percentage of contigs retained within each predicted cluster: \(100\%\) keeps all contigs, while \(10\%\) keeps only the \(10\%\) ranked as most reliable (cf.\ Equation~\ref{eq:local_reliability}). The curves compare different ways of performing this selection. We consider two variants based on consensus clustering: \emph{Consensus + reliability} uses the consensus clustering and ranks contigs by the reliability score \(\rho_i\), while \emph{consensus + random} uses the same clustering but rejects contigs randomly. As baseline, we use the SemiBin2 MAP-based clustering combined with three different ranking strategies based on geometric distance (\emph{MAP + centroid distance}), assignment margin (\emph{MAP + margin}), and pointwise posterior variance (\emph{MAP + LLA trace}). The horizontal lines report the HQ-cluster counts of the reference methods, namely DNABERT-S HQ \citep{Zhou2024DNABERTS}, RevisitKmers \citep{celikkanat2024revisiting} and UncertainGen \citep{celikkanat2025uncertaingen}. Besides showing as full horizontal lines, neither of these methods present a 100\% coverage, as all of them reject contigs during clustering, however, this coverage fraction is not reported in either of the works.

Across all four datasets, \emph{consensus + reliability} produces substantially more HQ clusters as coverage decreases than random rejection or the MAP-based heuristics, and surpasses the reference methods in every dataset except for Plant 5. Using the same reliability ranking with MAP assignments also improves over most pointwise heuristics, but remains well below the posterior-consensus construction. These results support, in the real-world metagenomic binning task, the theoretical results presented in this work: marginal or geometric uncertainty scores are weak proxies for a relational decision, whereas posterior agreement after applying the clustering rule directly captures whether the resulting groups remain stable under encoder perturbations.

\subsection{Image Classification}

For the image experiments, we apply the proposed approach to pretrained visual
representations and study how uncertainty in the learned embedding space
propagates to downstream classification and clustering. We consider CIFAR-100
and EMNIST using CLIP~\citep{radford2021learning} and
SimCLR~\citep{chen2020simple} representations, together with SICAPv2 prostate
histopathology patches~\citep{SILVARODRIGUEZ2020105637}, where we use CLIP and CONCH-v1.5, a vision--language foundation model for computational pathology~\citep{Lu2024CONCH}. The classification experiments
test whether posterior uncertainty over the representation improves predictive
uncertainty after a downstream readout, while clustering provides a relational
setting in which reliability depends on the stability of the groups induced by
the embedding space.

Again, we use each pretrained encoder as the MAP reference model and fit a GGN KFAC last-layer LLA posterior around this checkpoint~\citep{Martens2015KFAC}. For
classification, \textit{MAP} trains a multinomial logistic regressor on the
deterministic embeddings. \textit{Sampled LLA} samples encoders from the
posterior, trains the same logistic regressor in each sampled embedding space,
and averages the resulting predictive probabilities.
\textit{Gaussian LLA} places a KFAC Gaussian posterior over the parameters of the final linear embedding layer. The resulting uncertainty is propagated through the linear logistic classifier to obtain Gaussian logits. We train the classifier with Adam~\citep{kingma2015adam} using Monte Carlo estimates of the marginal predictive likelihood.

\begin{table*}[t]
  \centering
  \caption{Classification results on CIFAR-100, EMNIST, and SICAP. Values are mean $\pm$ one standard deviation over five seeds. The best result for each dataset and encoder is shown in bold. In the second encoder block, SimCLR is used for CIFAR-100 and EMNIST, and CONCH for SICAP.}
  \label{tab:figure2-classification}
  \setlength{\tabcolsep}{3.2pt}
  \renewcommand{\arraystretch}{0.95}
  \resizebox{\textwidth}{!}{%
  \begin{tabular}{llccc ccc ccc}
    \toprule
    & & \multicolumn{3}{c}{CIFAR-100} & \multicolumn{3}{c}{EMNIST} & \multicolumn{3}{c}{SICAP} \\
    \cmidrule(lr){3-5} \cmidrule(lr){6-8} \cmidrule(lr){9-11}
    Encoder & Method & Acc. $\uparrow$ & NLL $\downarrow$ & ECE $\downarrow$ & Acc. $\uparrow$ & NLL $\downarrow$ & ECE $\downarrow$ & Acc. $\uparrow$ & NLL $\downarrow$ & ECE $\downarrow$ \\
    \midrule
    CLIP & MAP & $74.00 \pm 0.00$ & $1.424 \pm 0.000$ & $0.165 \pm 0.000$ & $84.45 \pm 0.13$ & $0.533 \pm 0.010$ & $0.064 \pm 0.001$ & $64.44 \pm 0.04$ & $1.095 \pm 0.000$ & $0.175 \pm 0.000$ \\
     & LLA Sampled & $76.32 \pm 0.05$ & $0.993 \pm 0.004$ & $0.084 \pm 0.000$ & $\mathbf{85.63 \pm 0.11}$ & $0.454 \pm 0.007$ & $0.035 \pm 0.001$ & $66.70 \pm 0.45$ & $0.930 \pm 0.008$ & $0.123 \pm 0.004$ \\
     & Gaussian LLA & $\mathbf{77.34 \pm 0.21}$ & $\mathbf{0.766 \pm 0.002}$ & $\mathbf{0.014 \pm 0.003}$ & $84.95 \pm 0.10$ & $\mathbf{0.449 \pm 0.002}$ & $\mathbf{0.020 \pm 0.002}$ & $\mathbf{66.76 \pm 1.51}$ & $\mathbf{0.824 \pm 0.042}$ & $\mathbf{0.062 \pm 0.018}$ \\
    \midrule
    SimCLR  & MAP & $68.25 \pm 0.01$ & $1.774 \pm 0.012$ & $0.199 \pm 0.001$ & $\mathbf{81.79 \pm 0.29}$ & $0.686 \pm 0.017$ & $0.079 \pm 0.002$ & $79.90 \pm 0.03$ & $0.567 \pm 0.001$ & $0.093 \pm 0.001$ \\
     \quad \quad  / \quad & LLA Sampled & $\mathbf{69.65 \pm 0.12}$ & $1.451 \pm 0.004$ & $0.141 \pm 0.001$ & $81.74 \pm 0.41$ & $0.688 \pm 0.014$ & $0.080 \pm 0.004$ & $80.44 \pm 0.24$ & $0.545 \pm 0.004$ & $0.082 \pm 0.003$ \\
     CONCH & Gaussian LLA & $66.74 \pm 0.20$ & $\mathbf{1.188 \pm 0.002}$ & $\mathbf{0.107 \pm 0.003}$ & $79.82 \pm 0.57$ & $\mathbf{0.673 \pm 0.047}$ & $\mathbf{0.055 \pm 0.043}$ & $\mathbf{80.88 \pm 0.15}$ & $\mathbf{0.478 \pm 0.012}$ & $\mathbf{0.045 \pm 0.018}$ \\
    \bottomrule
  \end{tabular}%
  }
\end{table*}

Table~\ref{tab:figure2-classification} shows that posterior
propagation improves predictive uncertainty over the MAP
baseline in most cases. For CIFAR-100 and EMNIST, the LLA variants generally reduce NLL and
ECE for both CLIP and SimCLR representations, while the corresponding changes
in accuracy are smaller. The same behavior is observed on SICAP: CONCH provides
substantially stronger predictions than the general-purpose CLIP encoder,
reflecting the benefit of pathology-specific pretraining, but posterior
propagation still improves uncertainty quality for both representation
families. In particular, sampled and closed-form LLA generally reduce
NLL and ECE, with the latter producing the strongest overall calibration.
Thus, the main effect of posterior propagation is improved predictive
uncertainty rather than a large change in the MAP decisions.

\section{Discussion, Limitations, and Conclusion}

Contrastive objectives that behave similarly during MAP training can nevertheless induce different local Bayesian models. For this reason, the natural output variable of the loss, the PSD status of its output-Hessian blocks, and the precision spectrum before and after repair should be reported explicitly. These diagnostics help distinguish a valid relational LLA construction from a numerical procedure that produces a PSD matrix while obscuring the underlying loss geometry.

Our analysis is local to a trained solution, and the scalable experiments rely on last-layer or KFAC last-layer LLA rather than a full-network posterior. The proposed repairs yield valid Gaussian precisions, but they remain local approximations to objectives that are not always canonical likelihoods. They should therefore be interpreted as local modeling choices rather than as evidence that the original loss defines a well-specified likelihood. Similarly, reliable cores are task-specific summaries of uncertainty under a particular downstream rule and posterior, rather than universal uncertainty measures.

Overall, contrastive LLA raises two distinct questions. The first is geometric: does the loss provide valid local curvature? Logit-Bernoulli and InfoNCE do, whereas distance-Bernoulli and margin losses generally require repair. The second is operational: how should the resulting posterior be used downstream? Pointwise embedding variance and marginal pair-score variance are generally insufficient because they do not account for the structure of the downstream decision. Propagating the posterior through the downstream rule, for example by sampling repaired encoders and constructing reliable relational cores, instead yields uncertainty measures and quality--coverage tradeoffs that are directly tied to the decisions of interest.

\paragraph{Limitations.} (i) Our posterior should be interpreted as a task-adapted contrastive LLA posterior, rather than as a full Bayesian posterior over the original pretrained model. This distinction is particularly important for experiments based on CLIP, SimCLR, DNABERT-S-style, or CONCH embeddings, since the LLA is fitted after transfer using the available downstream representation and objective, rather than the original pretraining data and loss. Consequently, the posterior captures uncertainty in the transferred embedding/readout map under the fitted last-layer or KFAC approximation, but not all sources of uncertainty in the pretrained feature extractor. (ii) The proposed spectral and radial repairs ensure a valid Gaussian precision, but they also modify the local curvature of distance- and margin-based losses. They should therefore be viewed as local modeling policies rather than as corrections to a fully specified likelihood. (iii) Reliable cores are tied to a particular downstream rule and selection procedure. They measure whether clustering, classification, or grouping decisions remain stable under posterior encoder perturbations, but they are not universal uncertainty scores, calibrated out-of-distribution detectors, or guarantees of semantic correctness. In addition, posterior propagation requires either sampling encoders or propagating Gaussian embeddings through the downstream rule, making it more expensive than computing a single marginal variance score.

\subsection*{AI use statement}

In this work, we used generative AI tools to provide feedback on research methodology and experimental design and to assist with the interpretation and presentation of experimental results. We did not use generative AI tools to generate datasets, formulate mathematical claims, develop theoretical models, provide critical ingredients for or assist in writing proofs, propose hypotheses, implement the proposed methods, clean or reformat datasets, or perform qualitative or thematic data analysis. Translation and research tasks involving surveys or interviews were not applicable to this work.

Additionally, we used generative AI tools for brainstorming, searching for and summarizing potentially relevant literature, suggesting improvements to the structure and presentation of the paper, and drafting or editing text for clarity and readability. All methodological suggestions and interpretations produced with AI assistance were independently evaluated by the authors. All references were manually verified against their original sources, all experimental results and numerical claims were checked against the outputs of experiments conducted by the authors, and all AI-assisted text was reviewed and revised by the authors. We take responsibility for the final content of this work, including text, claims, and artifacts produced with the aid of generative AI.

\subsection*{Ethics statement}

This work develops uncertainty estimation methods using existing benchmark datasets and pretrained encoders. The SICAPv2 experiments use publicly available histopathology data with patient-separated training and test partitions; they provide a methodological evaluation and do not establish clinical suitability. Uncertainty estimates may inherit biases from the underlying data and pretrained representations, and posterior stability does not guarantee correctness. Furthermore, selecting reliable cores may disproportionately exclude rare or underrepresented examples. Applications should therefore assess both predictive performance and the effects of selective rejection across relevant classes and populations, particularly in sensitive domains.

\subsection*{Reproducibility statement}

Sections \ref{sec:background}-\ref{sec:decision-posteriors} describe the relational LLA construction, posterior propagation, and reliable-core selection. Appendices \ref{sec:geometry}-\ref{sec:appendix:marginal:uncertainty:proofs} provide additional curvature derivations, repair definitions, and analysis of the distinction between marginal variance and downstream uncertainty. Appendix~\ref{app:experiments} documents the datasets, preprocessing and data splits, encoder architectures, posterior target layers, contrastive objectives, hyperparameter selection, posterior sampling, downstream rules, and evaluation metrics. Experiments use five random seeds unless otherwise stated. Appendices~\ref{sec:additional-clustering-results}-\ref{app:ablations} report additional clustering results and ablations examining the sensitivity of the method to its modeling and downstream choices.



\bibliographystyle{plainnat}
\bibliography{references}

\newpage

\appendix

\section{Additional Related Work}
\label{app:related_work}

The Laplace approximation represents the posterior near a mode by a
Gaussian whose precision is determined by the local curvature of the
negative log posterior. In neural networks, this curvature is commonly
approximated or structured using generalized Gauss--Newton (GGN),
Fisher, Kronecker-factored, diagonal, low-rank, last-layer, or
subnetwork approximations
\citep{Daxberger2021LaplaceRedux,Martens2020NaturalGradient,
Kristiadi2020BeingBayesian, daxberger2021subnetwork}.

A useful modern interpretation is that GGN-based Laplace performs
Bayesian inference in a locally linearized generalized linear model,
rather than simply approximating the Hessian of the original nonlinear
network \citep{Immer2021LocalLinearization}. This view has motivated
several scalable Linearized Laplace Approximation (LLA) variants,
including variational formulations, low-rank feature approximations,
surrogate neural-kernel approximations, and quadratic corrections to
the local posterior precision
\citep{ortega2026uncertainty,Deng2022AcceleratedLLA,
Ortega2026SurrogateNeuralKernel,Jimenez2026QuadraticLLA}.
Related work has also enriched Bayesian last layers using subsampled
neural-tangent features, retaining a tractable local linear model while
providing greater expressivity than standard last-layer Laplace
\citep{CalvoOrdonez2026RicherBayesianLastLayers}.

Beyond Euclidean Gaussian approximations, Riemannian and tubular
Riemannian Laplace methods adapt posterior geometry to curved,
anisotropic, and symmetry-induced low-loss regions of the neural-network
loss landscape
\citep{Bergamin2023RiemannianLaplace,David2025TubularRiemannianLaplace}.
These approaches have primarily been developed for supervised
likelihoods that decompose over individual examples. In contrast, our
setting applies LLA to relational objectives, where each loss term
depends jointly on pairs or sets of network outputs and the relevant
curvature therefore lives in the output space of a contrastive unit.

\section{Contrastive Curvature Pathologies and Repairs}
\label{sec:geometry}

This appendix develops the loss-level details of the relational LLA
construction in Section~\ref{sec:background}. For a relational unit $t$,
the linearized Laplace precision contains blocks of the form
$\mathbf{J}_t^\top\mathbf{H}_t\mathbf{J}_t$, where
$\mathbf{H}_t=\nabla^2_{f_t}\ell_t(f_t)$ and $f_t$ is the output
variable used by the contrastive loss. The choice of this variable
matters. Logit losses define curvature with respect to a scalar score
or a vector of candidate logits, whereas distance losses define it
with respect to an embedding difference. Although these choices can
have little visible effect during training, they can lead to very
different local Bayesian models.

Throughout this section, we write
$\mathbf{z}_i=h_{\boldsymbol{\theta}}(\mathbf{x}_i)$,
$s_{ij}=\tau^{-1}\mathbf{z}_i^\top\mathbf{z}_j$, and
$\mathbf{v}_{ij}=\mathbf{z}_i-\mathbf{z}_j$, where $\tau>0$ is the
temperature. All statements about curvature refer to the output space,
before adding the prior, damping, or numerical jitter. We use $y=1$
for a positive pair and $y=0$ for a negative pair.

\subsection{Logit losses give PSD output curvature}

\paragraph{Logit-Bernoulli.}
For a scalar pair score $s_{ij}$, the Bernoulli negative log-likelihood
is $\llogit(s_{ij},y)=\operatorname{softplus}(s_{ij})-ys_{ij}$.
Writing $p=\sigma(s_{ij})$, its first derivative is
$\partial\llogit/\partial s_{ij}=p-y$, and its second derivative is
$\partial^2\llogit/\partial s_{ij}^2=p(1-p)\ge0$. Each output
curvature block is therefore PSD, as is the corresponding GGN
contribution $p(1-p)\mathbf J_t^\top\mathbf J_t$, regardless of the
Jacobian $\mathbf J_t$.

\paragraph{InfoNCE.}
For an anchor $i$ with candidate set $C_i$ and positive candidate
$p(i)\in C_i$, the output is the vector of candidate logits
$f_t=(a_{t,j})_{j\in C_i}$, with $a_{t,j}=s_{ij}$. The InfoNCE loss is
\[
    \lnce(f_t)
    =-a_{t,p(i)}+\log\sum_{j\in C_i}\exp(a_{t,j}).
\]
If $\mathbf q_i=\operatorname{softmax}(f_t)$, the gradient is
$\nabla_{f_t}\lnce=\mathbf q_i-\mathbf e_{p(i)}$, and the Hessian is
\[
    \nabla^2_{f_t}\lnce
    =\operatorname{diag}(\mathbf q_i)-\mathbf q_i\mathbf q_i^\top.
\]
This is the covariance matrix of the one-hot representation of a
categorical variable with probabilities $\mathbf q_i$. Equivalently,
for any vector $\mathbf r$,
\[
    \mathbf r^\top
    \bigl(\operatorname{diag}(\mathbf q_i)-\mathbf q_i\mathbf q_i^\top\bigr)
    \mathbf r
    =\sum_{j\in C_i}q_{i,j}r_j^2
    -\left(\sum_{j\in C_i}q_{i,j}r_j\right)^2
    =\operatorname{Var}_{j\sim\mathbf q_i}(r_j)\ge0.
\]
InfoNCE therefore also gives a PSD output curvature block. Its rank is
at most $|C_i|-1$, since adding the same constant to every logit leaves
the softmax unchanged. This rank deficiency is not a pathology: the
prior and numerical jitter make the final posterior precision
invertible.

\subsection{Distance losses can give indefinite output curvature}

For distance-based losses, the output variable is the embedding
displacement $\mathbf v=\mathbf v_{ij}=\mathbf z_i-\mathbf z_j$.
We therefore need the curvature with respect to $\mathbf v$, rather
than a scalar logit, and it is in this space that negative curvature
can arise.

To see where it comes from, consider a radial loss
$\ell(\mathbf v)=g(r)$, where $r=\|\mathbf v\|_2>0$ and
$\mathbf u=\mathbf v/r$. Its Hessian is
\[
    \nabla_{\mathbf v}^2\ell(\mathbf v)
    =g''(r)\mathbf u\mathbf u^\top
    +\frac{g'(r)}{r}
    \bigl(\mathbf I_d-\mathbf u\mathbf u^\top\bigr).
\]
There is one radial eigenvalue, along $\mathbf u$, and $d-1$
tangential eigenvalues, in directions orthogonal to $\mathbf u$.
Tangential perturbations change the orientation of the difference
vector to first order. A negative tangential eigenvalue means that
the loss is locally concave along such a direction, so a small
straight-line displacement in that direction decreases the loss.
This negative curvature cannot be used as Gaussian precision.

\paragraph{Distance-Bernoulli.}
The distance-Bernoulli model sets
$p(\mathbf v)=\exp(-\|\mathbf v\|_2^2)$. For a positive pair
($y=1$), the loss is
$\ldist_+(\mathbf v)=-\log p(\mathbf v)=\|\mathbf v\|_2^2$,
with Hessian $\nabla_{\mathbf v}^2\ldist_+(\mathbf v)=2\mathbf I_d\succeq0$.
Positive pairs therefore do not cause a problem with positive
definiteness.

For a negative pair ($y=0$), the loss becomes
\[
    \ldist_-(\mathbf v)
    =-\log\bigl(1-\exp(-\|\mathbf v\|_2^2)\bigr).
\]
Writing $p=\exp(-\|\mathbf v\|_2^2)$ and $q=1-p$, its Hessian
for $\mathbf v\ne\mathbf0$ is
\begin{equation}
\label{eq:distance_negative_hessian_main}
    \mathbf H_-^{\mathrm{dist}}(\mathbf v)
    =\nabla_{\mathbf v}^2\ldist_-(\mathbf v)
    =4\frac{p}{q^2}\mathbf v\mathbf v^\top
    -2\frac{p}{q}\mathbf I_d.
\end{equation}
The radial eigenvalue is
\[
    \lambda_{\mathrm{dist,rad}}
    =4\frac{p}{q^2}\|\mathbf v\|_2^2-2\frac{p}{q}
    =2\frac{p}{q^2}\bigl(2\|\mathbf v\|_2^2-q\bigr),
\]
which is positive whenever $\mathbf v\ne\mathbf0$. Every tangential
eigenvalue, however, is $\lambda_{\mathrm{dist,tan}}=-2p/q<0$.
The output Hessian of a negative distance-Bernoulli pair is therefore
indefinite when the embedding dimension is $d\ge2$.

\paragraph{Pairwise margin loss.}
The standard pairwise margin loss is
\[
    \lmrg(\mathbf v,y)
    =y\|\mathbf v\|_2^2+(1-y)[m-\|\mathbf v\|_2]_+^2.
\]
For positive pairs, we again have
$\lmrg_+(\mathbf v)=\|\mathbf v\|_2^2$ and
$\nabla_{\mathbf v}^2\lmrg_+(\mathbf v)=2\mathbf I_d$, so the
Hessian is PSD. Inactive negatives, with $\|\mathbf v\|_2\ge m$,
are assigned zero curvature, with the boundary case discussed below.
The problematic case is an active negative with
$0<r=\|\mathbf v\|_2<m$. Its loss is
$\lmrg_-(\mathbf v)=(m-r)^2$, and its Hessian is
\begin{equation}
\label{eq:margin_negative_hessian_main}
    \mathbf H_-^{\mathrm{mrg}}(\mathbf v)
    =2\mathbf u\mathbf u^\top
    +\frac{2(r-m)}{r}
    \bigl(\mathbf I_d-\mathbf u\mathbf u^\top\bigr),
    \qquad \mathbf u=\mathbf v/r.
\end{equation}
The radial eigenvalue is $2$, but every tangential eigenvalue is
$\lambda_{\mathrm{mrg,tan}}=2(r-m)/r<0$. Active negative margin
pairs thus also have indefinite Hessians whenever $d\ge2$.

At $r=0$ and $r=m$, the classical Hessian is not defined. In practice,
these points are either avoided almost surely or treated using
standard numerical stabilizers, such as replacing $r$ by
$\sqrt{\|\mathbf v\|_2^2+\epsilon}$ in diagnostic computations.
The negative curvature does not arise from these exceptional points:
it is present throughout the active region $0<r<m$.

\begin{proposition}
\label{prop:visible-negative-curvature}
Let $\mathbf H_t$ be the output Hessian of a training unit and
$\mathbf J_t$ its last-layer Jacobian. If a parameter perturbation
$\mathbf u$ satisfies
$(\mathbf J_t\mathbf u)^\top\mathbf H_t(\mathbf J_t\mathbf u)<0$,
then $\mathbf J_t^\top\mathbf H_t\mathbf J_t$ is not positive
semidefinite. In particular, consider a negative distance-Bernoulli
pair with $\mathbf v\ne\mathbf0$, or an active margin negative with
$0<\|\mathbf v\|_2<m$. When $d\ge2$, any nonzero tangential
direction $\mathbf w\perp\mathbf v$ that can be written as
$\mathbf w=\mathbf J_t\mathbf u$ gives negative curvature.
\end{proposition}

\begin{proof}
For any parameter perturbation $\mathbf u$,
\[
    \mathbf u^\top\mathbf J_t^\top\mathbf H_t\mathbf J_t\mathbf u
    =(\mathbf J_t\mathbf u)^\top\mathbf H_t(\mathbf J_t\mathbf u).
\]
If the right-hand side is negative, the parameter-space block has a
negative quadratic form in the direction $\mathbf u$ and therefore
cannot be PSD.
\end{proof}

Linearizing the encoder therefore does not, by itself, resolve the
problem with distance losses. LLA removes second derivatives of the
network or relational score map, but still uses the output Hessian of
the loss. When this Hessian is indefinite, the corresponding GGN block
can be indefinite too. Contributions from other training units, the
prior, or sufficiently large damping may hide this in the assembled
precision, but the loss itself no longer guarantees a valid Gaussian
precision.

\subsection{PSD output-curvature repairs}

Because the negative curvature arises within the output Hessian of a
contrastive unit, we repair each block before accumulating it in
parameter space. We keep the MAP encoder and relational Jacobian fixed
and replace only $\mathbf H_t$ by a PSD block
$\widetilde{\mathbf H}_t\succeq0$. This gives the repaired precision
\begin{equation}
\label{eq:repaired_precision_main}
    \widetilde{\boldsymbol{\Sigma}}^{-1}
    =\boldsymbol{\Lambda}_0
    +\sum_t\mathbf J_t^\top\widetilde{\mathbf H}_t\mathbf J_t
    +\epsilon\mathbf I,
\end{equation}
where $\boldsymbol{\Lambda}_0$ includes the prior or weight-decay
precision and $\epsilon>0$ is a small numerical jitter. Under the
default spectral repair described below, blocks that are already PSD,
including those from logit-Bernoulli and InfoNCE, are left unchanged
up to numerical tolerance.

A repair makes the local Gaussian precision valid, but its
interpretation remains local. It does not retrain the encoder or
change the MAP embeddings, nor does it establish that the original
distance or margin loss was a fully specified likelihood. It is a
modeling choice that turns an indefinite contrastive curvature block
into a usable PSD precision block. We consider two such repairs.

\paragraph{Spectral repair.}
Our default repair is the spectral projection onto the PSD cone,
which can be applied regardless of the loss. We first remove any
numerical asymmetry by setting
$\mathbf H_{t,\mathrm{sym}}=\tfrac12(\mathbf H_t+\mathbf H_t^\top)$.
If its eigendecomposition is
$\mathbf H_{t,\mathrm{sym}}=\mathbf Q\operatorname{diag}(\lambda_1,\ldots,\lambda_m)\mathbf Q^\top$,
the repaired block is
\begin{equation}
\label{eq:spectral_repair}
    \widetilde{\mathbf H}_t^{\mathrm{spec}}
    =\Pi_+(\mathbf H_t)
    =\mathbf Q\operatorname{diag}([\lambda_1]_+,\ldots,[\lambda_m]_+)\mathbf Q^\top,
    \qquad [a]_+=\max\{a,0\}.
\end{equation}
This retains every nonnegative curvature direction supplied by the
loss and removes the negative directions that could make the Gaussian
precision indefinite. It differs from the absolute-value transform
$\lambda_i\mapsto|\lambda_i|$: negative curvature is discarded,
rather than converted into positive precision. The repair can
therefore reduce the rank of a block. The prior and jitter in
\eqref{eq:repaired_precision_main} make the final precision invertible.

\paragraph{Radial-only repair for distance-based losses.}
For a distance-based loss, the output is the embedding difference
$\mathbf v$, and its eigenspaces are available analytically: the
radial direction $\mathbf u=\mathbf v/\|\mathbf v\|_2$ and the
$d-1$ tangential directions orthogonal to it. Let
$\mathbf P_{\mathrm{rad}}=\mathbf u\mathbf u^\top$ and
$\mathbf P_{\mathrm{tan}}=\mathbf I_d-\mathbf u\mathbf u^\top$.
For a radial Hessian
$\mathbf H(\mathbf v)=\lambda_{\mathrm{rad}}\mathbf P_{\mathrm{rad}}
+\lambda_{\mathrm{tan}}\mathbf P_{\mathrm{tan}}$,
spectral repair retains all nonnegative eigenvalues:
\[
    \widetilde{\mathbf H}_{\mathrm{spec}}(\mathbf v)
    =[\lambda_{\mathrm{rad}}]_+\mathbf P_{\mathrm{rad}}
    +[\lambda_{\mathrm{tan}}]_+\mathbf P_{\mathrm{tan}}.
\]
The radial-only repair used in our experiments goes further. It
retains only the radial curvature, which describes changes in
distance, and applies damping in the tangential subspace:
\[
    \widetilde{\mathbf H}^{\mathrm{rad}}(\mathbf v)
    =[\lambda_{\mathrm{rad}}]_+\mathbf P_{\mathrm{rad}}
    +\epsilon_{\mathrm{tan}}\mathbf P_{\mathrm{tan}}.
\]
Here $\epsilon_{\mathrm{tan}}\ge0$ is a small tangential damping
value, set to zero unless stated otherwise. This is a structured
approximation, rather than the PSD projection of the Hessian, because
it discards tangential curvature even when that curvature is positive.

For negative distance-Bernoulli pairs, write
$\lambda_{\mathrm{dist,rad}}=4(p/q^2)\|\mathbf v\|_2^2-2p/q$,
where $p=\exp(-\|\mathbf v\|_2^2)$ and $q=1-p$. The radial-only
repaired block is then
\[
    \widetilde{\mathbf H}_-^{\mathrm{dist,rad}}(\mathbf v)
    =[\lambda_{\mathrm{dist,rad}}]_+\mathbf u\mathbf u^\top
    +\epsilon_{\mathrm{tan}}\bigl(\mathbf I_d-\mathbf u\mathbf u^\top\bigr).
\]
For active margin negatives, the radial eigenvalue is $2$ and the
tangential eigenvalues are negative, giving
\[
    \widetilde{\mathbf H}_-^{\mathrm{mrg,rad}}(\mathbf v)
    =2\mathbf u\mathbf u^\top
    +\epsilon_{\mathrm{tan}}\bigl(\mathbf I_d-\mathbf u\mathbf u^\top\bigr),
    \qquad 0<\|\mathbf v\|_2<m.
\]
The loss is locally flat for inactive margin negatives away from the
boundary. We set
$\widetilde{\mathbf H}_-^{\mathrm{mrg,rad}}(\mathbf v)=\mathbf0$
for $\|\mathbf v\|_2\ge m$, including the boundary convention
described above.

The difference between spectral and radial-only repair is clearest
for positive pairs. Both distance and margin losses have
\[
    \mathbf H_+^{\mathrm{dist}}(\mathbf v)
    =\mathbf H_+^{\mathrm{mrg}}(\mathbf v)
    =2\mathbf I_d
    =2\mathbf P_{\mathrm{rad}}+2\mathbf P_{\mathrm{tan}}.
\]
Spectral repair leaves these PSD blocks unchanged, so
$\widetilde{\mathbf H}_+^{\mathrm{dist,spec}}(\mathbf v)
=\widetilde{\mathbf H}_+^{\mathrm{mrg,spec}}(\mathbf v)=2\mathbf I_d$.
Radial-only repair instead retains the radial component and damps
the tangential component:
\[
    \widetilde{\mathbf H}_+^{\mathrm{dist,rad}}(\mathbf v)
    =\widetilde{\mathbf H}_+^{\mathrm{mrg,rad}}(\mathbf v)
    =2\mathbf u\mathbf u^\top
    +\epsilon_{\mathrm{tan}}\bigl(\mathbf I_d-\mathbf u\mathbf u^\top\bigr).
\]

For active negative margin pairs, the two repairs therefore agree up
to the small tangential damping term. For positive pairs, spectral
repair preserves the full PSD Hessian $2\mathbf I_d$, while
radial-only repair removes its valid positive tangential curvature.
This makes radial-only repair more intrusive in Frobenius norm. We
use it as an ablation or modeling choice that retains only evidence
about changes in distance, rather than as the nearest PSD projection
of the original Hessian. It applies specifically to radial distance
losses and should not be used for candidate-logit losses such as
InfoNCE.

\paragraph{Damping and jitter are not curvature repairs.}
Adding $\lambda\mathbf I$, or increasing the prior precision, can
make the assembled parameter-space precision positive definite. In
particular, for sufficiently large $\lambda$,
\[
    \boldsymbol{\Lambda}_0
    +\sum_t\mathbf J_t^\top\mathbf H_t\mathbf J_t
    +\lambda\mathbf I\succ0.
\]
Although this can help numerically, it does not repair the
contrastive output curvature: sufficiently large isotropic precision
can simply overwhelm the negative-curvature blocks. In this work,
we use damping and jitter to stabilize inversion after repairing the
output curvature. They are not our primary means of obtaining a
valid posterior from an indefinite loss geometry.

\begin{table}[t]
    \centering
    \small
    \setlength{\tabcolsep}{4pt}
    \renewcommand{\arraystretch}{1.05}
    \begin{tabular}{@{}p{2.2cm}p{4.2cm}p{4.2cm}p{3.4cm}@{}}
        \toprule
        Mode & Construction & What it does & What it does not do \\
        \midrule
        Raw GGN &
        Use $\mathbf H_t=\nabla^2_{f_t}\ell_t$ directly &
        Retains the exact output Hessian of the linearized loss &
        Does not guarantee a PSD precision for distance or margin negatives \\
        Spectral repair &
        Symmetrize, eigendecompose, and set negative eigenvalues to zero &
        Projects each output block onto the PSD cone, regardless of the loss &
        Removes negative curvature; does not turn it into positive evidence \\
        Radial-only repair &
        Split distance-loss blocks into radial and tangential components;
        retain radial curvature and damp the tangential component &
        Gives a structured PSD approximation for distance losses; agrees with
        spectral repair on active negative margin blocks up to damping &
        Is not the PSD projection and can remove valid positive tangential
        curvature from already-PSD positive-pair blocks \\
        Damping/jitter &
        Add isotropic precision after accumulation &
        Stabilizes inversion and can make the final matrix strictly positive definite &
        Does not identify or repair the loss-level curvature pathology \\
        Empirical Fisher &
        Replace $\mathbf H_t$ by $\mathbf g_t\mathbf g_t^\top$ &
        Gives a PSD block by construction &
        Does not repair the Hessian and can reduce relational curvature to
        rank one per unit \\
        \bottomrule
    \end{tabular}
    \caption{Repair and fallback modes considered in this work. The main
    experiments use repaired GGN curvature, while empirical Fisher and
    alternative repairs are used as ablations.}
    \label{tab:repair_modes}
\end{table}

Logit-Bernoulli and InfoNCE give PSD output curvature directly,
whereas negative pairs in distance-Bernoulli and pairwise margin
losses can have negative tangential curvature. Spectral repair
projects these indefinite blocks onto the PSD cone before they are
accumulated in parameter space. Radial-only repair is a structured
approximation for distance losses: it retains curvature associated
with changes in distance and damps the tangential, rotation-like
directions. It can also remove positive tangential curvature, as it
does for positive distance and margin pairs. Radial-only repair
should therefore be understood as a modeling ablation that retains
radial evidence, rather than as the nearest PSD projection of the
Hessian.

\section{Empirical Fisher Alternative for Relational Curvature}
\label{app:ef}

Our main construction uses the output Hessian
$\mathbf H_t^{\mathrm{GGN}}=\nabla^2_{f_t}\ell_t(f_t)$ in the
relational GGN block
$\mathbf J_t^\top\mathbf H_t^{\mathrm{GGN}}\mathbf J_t$.
When the output Hessian is indefinite, as it is for negative
distance-Bernoulli pairs or active negative margin pairs, we project
it onto the PSD cone before accumulating the resulting blocks in
parameter space. The empirical Fisher (EF) offers an alternative
that does not require Hessian blocks.

For a relational unit $t$, let
$\mathbf g_t=\left.\nabla_{f_t}\ell_t(f_t)\right|_{f_t=f_t(\boldsymbol{\theta}^\star)}$
be the output gradient at the MAP solution. EF replaces the output
Hessian with the outer product of this gradient:
\begin{equation}
\label{eq:ef_output_block}
    \mathbf H_t^{\mathrm{EF}}=\mathbf g_t\mathbf g_t^\top.
\end{equation}
The corresponding parameter-space precision is
\begin{equation}
\label{eq:ef_precision}
    \boldsymbol{\Sigma}_{\mathrm{EF}}^{-1}
    =\boldsymbol{\Lambda}_0
    +\sum_t\mathbf J_t^\top\mathbf H_t^{\mathrm{EF}}\mathbf J_t
    =\boldsymbol{\Lambda}_0
    +\sum_t(\mathbf J_t^\top\mathbf g_t)
    (\mathbf J_t^\top\mathbf g_t)^\top.
\end{equation}
Each contribution is an outer product and is therefore PSD. With a
strictly positive definite prior precision or positive damping, the
resulting precision defines a proper Gaussian posterior approximation.

\paragraph{Why EF is attractive.}
EF has two practical advantages. It requires neither an
eigendecomposition of the output Hessian nor a separate repair for
distance or margin losses. It also applies in the same way to scalar
scores, candidate-logit vectors, and embedding differences. This
makes it a convenient fallback when the output Hessian is difficult
to form, affected by nonsmoothness, or indefinite.

\paragraph{EF is not a Hessian repair.}
Although EF is PSD, it does not project the Hessian onto the PSD
cone. Instead, it replaces second-order curvature with first-order
sensitivity at the MAP solution, giving the posterior a different
local interpretation. Spectrally repaired GGN retains the
nonnegative curvature directions of the loss and removes the
negative directions that cannot serve as precision. EF retains only
the current gradient direction for each relational unit.

For relational contrastive objectives, this is a substantial
difference. Each unit contributes a block of rank at most one,
$\operatorname{rank}(\mathbf H_t^{\mathrm{EF}})\le1$. The sum over
many units can still have high rank, but an individual unit supplies
only one direction. The GGN Hessian of a candidate-set loss such as
InfoNCE, by comparison, can describe a full covariance structure
over the candidates within that unit.

\paragraph{Relation to true Fisher and GGN.}
For well-specified likelihood models, and under the usual regularity
conditions, the expected Fisher can coincide with the expected
Hessian of the negative log-likelihood. The empirical Fisher in
\eqref{eq:ef_output_block} instead uses the observed target for each
training unit to form a single gradient outer product. For example,
consider the scalar logit-Bernoulli loss
$\ell(s,y)=\operatorname{softplus}(s)-ys$, with $p=\sigma(s)$.
Its GGN curvature is $\mathbf H^{\mathrm{GGN}}=p(1-p)$, whereas
its EF curvature is $\mathbf H^{\mathrm{EF}}=(p-y)^2$. Both are
nonnegative, but they measure different quantities.

Similarly, for InfoNCE with candidate probabilities
$\mathbf q=\operatorname{softmax}(f_t)$ and positive candidate
$p(i)$, the two blocks are
\[
    \mathbf H_t^{\mathrm{GGN}}
    =\operatorname{diag}(\mathbf q)-\mathbf q\mathbf q^\top,
    \qquad
    \mathbf H_t^{\mathrm{EF}}
    =(\mathbf q-\mathbf e_{p(i)})
    (\mathbf q-\mathbf e_{p(i)})^\top.
\]
The GGN block gives the softmax covariance over candidates. The EF
block is an outer product of the current prediction error, with rank
at most one, and therefore loses the multi-candidate interaction
geometry retained by the GGN.

\paragraph{Distance and margin losses.}
The same distinction arises for distance-based losses. As before,
write $\mathbf v=\mathbf z_i-\mathbf z_j$, $r=\|\mathbf v\|_2$,
and $\mathbf u=\mathbf v/r$ for $r>0$. A positive pair with
squared-distance loss $\ell^+(\mathbf v)=\|\mathbf v\|_2^2$
has GGN/Hessian $\mathbf H_+^{\mathrm{GGN}}=2\mathbf I_d$,
but its EF block is
$\mathbf H_+^{\mathrm{EF}}=(2\mathbf v)(2\mathbf v)^\top
=4\mathbf v\mathbf v^\top$. Thus, even when the positive-pair
Hessian is already PSD, EF replaces isotropic attraction curvature
with a single radial direction.

For an active negative margin pair, the loss is
$\lmrg_-(\mathbf v)=(m-r)^2$, with $0<r<m$. The raw Hessian has
radial eigenvalue $2$ and negative tangential eigenvalues. Both
spectral and radial-only repair remove these negative tangential
directions and retain the radial curvature, with radial-only repair
optionally adding tangential damping:
\[
    \widetilde{\mathbf H}^{\mathrm{GGN}}_{\mathrm{mrg}}
    =2\mathbf u\mathbf u^\top
    \quad\text{or}\quad
    2\mathbf u\mathbf u^\top
    +\epsilon_{\mathrm{tan}}
    (\mathbf I_d-\mathbf u\mathbf u^\top).
\]
Apart from this damping term, repaired GGN and EF are both PSD and
radial, but they assign different precision along the radial
direction. Repaired GGN retains the local second derivative of the
active margin loss. EF uses the squared first derivative, so its
precision vanishes as $r$ approaches the margin.

For a negative distance-Bernoulli pair, the loss is
$\ldist_-(\mathbf v)=-\log(1-\exp(-\|\mathbf v\|_2^2))$.
Writing $p=\exp(-r^2)$ and $q=1-p$, its gradient is
$\nabla_{\mathbf v}\ldist_-=-2(p/q)\mathbf v$, giving the EF block
\[
    \mathbf H^{\mathrm{EF}}_{\mathrm{dist}}
    =4\frac{p^2}{q^2}\mathbf v\mathbf v^\top.
\]
This block is again PSD and radial, but it is not the PSD projection
of the Hessian. Spectral repair retains the positive radial Hessian
eigenvalue and removes the negative tangential eigenvalues, whereas
EF replaces the Hessian with a squared-gradient magnitude.

\paragraph{What EF does and does not provide.}
EF provides a simple PSD approximation, but it should not be taken
to preserve the relational geometry of the contrastive loss. It
produces PSD output blocks by construction, avoids explicit Hessian
eigendecompositions and loss-specific repairs, and gives a simple
fallback curvature for all relational units.

It does not, however, project the true Hessian onto the PSD cone or
preserve high-rank candidate interactions such as the InfoNCE
softmax covariance. It also discards tangential or isotropic
curvature that may be present in PSD Hessian blocks, and it does not
distinguish invalid negative curvature from valid positive curvature.

These properties make EF useful as a PSD fallback and as an
ablation, but it changes the geometry of the local Bayesian model
more substantially than spectral or radial repair. We therefore use
repaired GGN curvature as the default in our main experiments.

We repeat the metagenomic binning experiment from
Section~\ref{sec:experiment:metagenomic}, whose original results
appear in Figure~\ref{fig:genome-reliable-cores}, using EF in place
of GGN. Figure~\ref{fig:genome-reliable-cores:ef} shows the results.

\begin{figure}[t]
    \centering
    \includegraphics[width=0.98\linewidth]{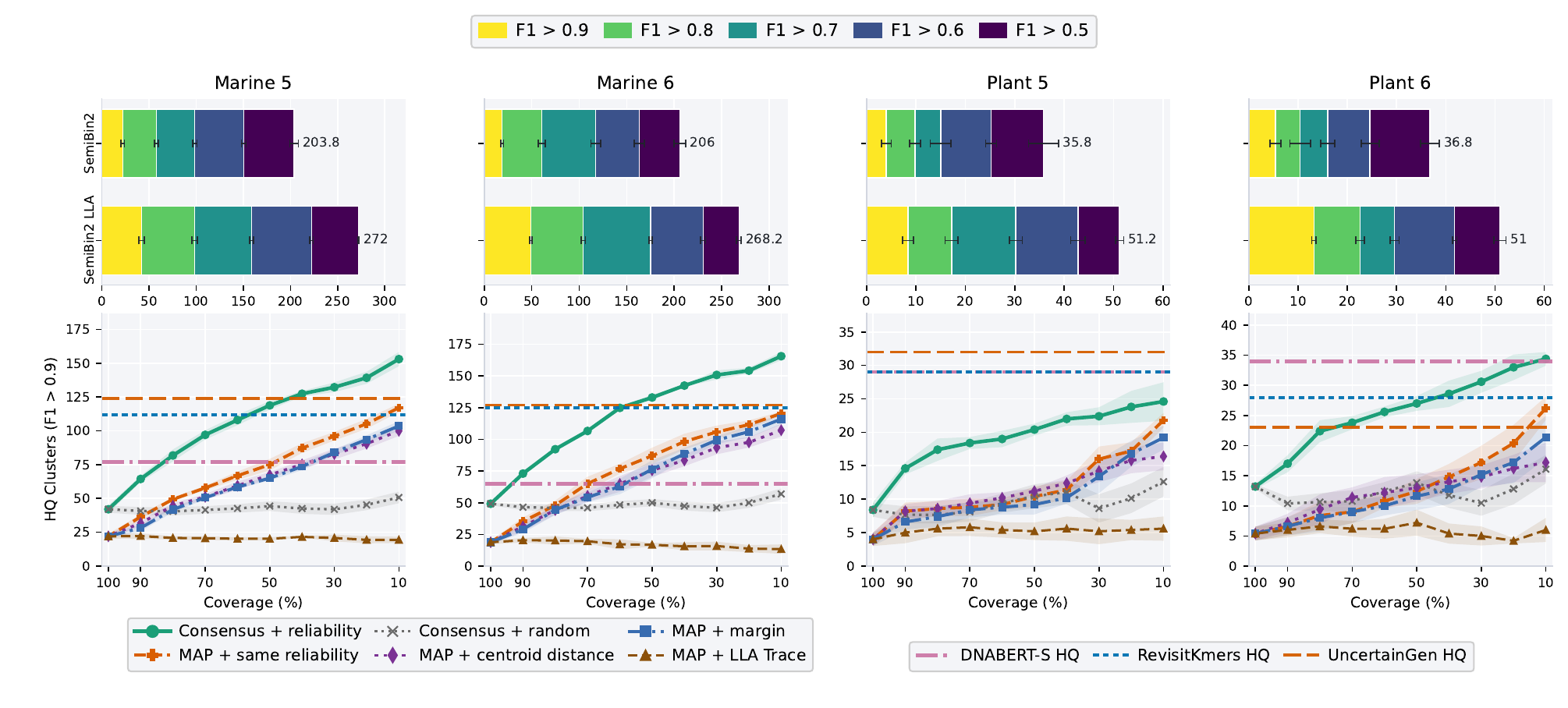}
    \caption{Genome reliable-core results on the Marine and Plant
    splits using EF in place of GGN, with $k=500$ retained for
    k-Means. The top panels compare high-quality cluster counts
    across F1 thresholds for SemiBin2 and the LLA posterior-consensus
    variant. The bottom panels show the number of high-quality
    clusters retained as coverage decreases under the different
    approaches.}
    \label{fig:genome-reliable-cores:ef}
\end{figure}

\section{Additional details on marginal variance and downstream uncertainty}
\label{sec:appendix:marginal:uncertainty:proofs}

In this section, we provide proofs for the propositions in Section \ref{sec:variance}. Also, to provide with visual intuition, we show how these results apply in a very simple example in Figure \ref{fig:assignment_uncertainty} and also for a pairwise score in Figure \ref{fig:score_rank_uncertainty}.

\textbf{Proposition 1. }Let \(h_{\bm{\theta}}\) be a (linearized) contrastive encoder and \(q(\bm{\theta})\) be a (repaired) LLA posterior, then, for any \(\mathbf{x} \in \mathcal{X}\), 
\(
    h_{\bm{\theta}}(\mathbf{x})\mid\mathcal{D}
    \sim
    \mathcal{N}(\boldsymbol{\mu},\boldsymbol{\Sigma})
\). 
Consider a linear binary decision boundary defined as
\(
    D(\mathbf{z})
    :=
    \mathbf{1}\{\mathbf{a}^{\top}\mathbf{z}+b\ge 0\}
\).
Then, denoting the standard gaussian c.d.f. as \(\Phi\), the posterior probability of changing the assignment made at the posterior mean is
\[
    \mathbb{P}\!\left[
        D(h(\mathbf{x}))\neq D(\boldsymbol{\mu})
        \mid\mathcal{D}
    \right]
    =
    \Phi\!\left(
        -
        \frac{
            |\mathbf{a}^{\top}\boldsymbol{\mu}+b|
        }{
            \sqrt{\mathbf{a}^{\top}\boldsymbol{\Sigma}\mathbf{a}}
        }
    \right)\,.
\]

\begin{proof}
Since an affine transformation of a Gaussian vector is Gaussian,
\[
\mathbf{a}^\top h_{\bm{\theta}}(\mathbf{x})+b
\mid\mathcal{D}
\sim
\mathcal{N}\left(
\mathbf{a}^\top\bm{\mu}+b,\,
\mathbf{a}^\top\bm{\Sigma}\mathbf{a}
\right).
\]
Write $m=\mathbf{a}^\top\bm{\mu}+b$ and
$s^2=\mathbf{a}^\top\bm{\Sigma}\mathbf{a}$, and assume $s>0$.
If $m\geq 0$, then $D(\bm{\mu})=1$, so the assignment changes
when the Gaussian score falls below zero. This has probability
$\Phi(-m/s)$. If $m<0$, then $D(\bm{\mu})=0$, and the assignment
changes when the score is nonnegative, which has probability
$1-\Phi(-m/s)=\Phi(m/s)$.

In both cases, the disagreement probability is $\Phi(-|m|/s)$.
Thus,
\[
\mathbb{P}\left[
D(h_{\bm{\theta}}(\mathbf{x}))\neq D(\bm{\mu})
\mid\mathcal{D}
\right]
=
\Phi\left(
-\frac{\left|\mathbf{a}^\top\bm{\mu}+b\right|}
{\sqrt{\mathbf{a}^\top\bm{\Sigma}\mathbf{a}}}
\right).
\]
If $s=0$, the score equals $m$ almost surely, so the
disagreement probability is zero.
\end{proof}

\textbf{Proposition 2. }
Suppose an LLA posterior \(q(\bm{\theta}) = \mathcal{N}(\bm{\theta}^\star, \bm{\Sigma})\), the two posterior (linearized) scores are jointly Gaussian under
\[
\begin{pmatrix}
    s_{\btheta}(\mathbf{x}_i,\mathbf{x}_j)\\
    s_{\btheta}(\mathbf{x}_i,\mathbf{x}_k)
\end{pmatrix} 
\sim 
\mathcal{N}\left(
    \begin{pmatrix}
        \mu_{ij}\\
        \mu_{ik}
    \end{pmatrix},
    \begin{pmatrix}
        V_{ij}     & C_{ij,ik}\\
        C_{ik, ij} & V_{ik}
    \end{pmatrix}
 \right)
\]
For the downstream comparison
\(
    A_i
    =
    \mathbf{1}
    \left\{
        s_{\btheta}(\mathbf{x}_i,\mathbf{x}_j)
        \ge
        s_{\btheta}(\mathbf{x}_i,\mathbf{x}_k)
    \right\}
\).
The probability of changing the score rank w.r.t the MAP is given by
\[
    \mathbb{P}
    \left[
        A_i(\btheta)
        \neq
        A_i(\btheta^{\star})
    \right]
    =
    \Phi\!\left(
        -
        \frac{
            |\mu_{ij}-\mu_{ik}|
        }{
            \sqrt{V_{ij}+V_{ik}-2C_{ij,ik}}
        }
    \right),
\]
\begin{proof}
The two scores are jointly Gaussian, so their difference
\[
\Delta_{\bm{\theta}}
=
s_{\bm{\theta}}(\mathbf{x}_i,\mathbf{x}_j)
-
s_{\bm{\theta}}(\mathbf{x}_i,\mathbf{x}_k)
\]
is Gaussian with mean $m=\mu_{ij}-\mu_{ik}$ and variance
$v=V_{ij}+V_{ik}-2C_{ij,ik}$.

The comparison is determined by the sign of this difference:
\[
A_i(\bm{\theta})
=
\mathbf{1}\{\Delta_{\bm{\theta}}\geq 0\}.
\]
Since the linearized posterior is centered at $\bm{\theta}^\star$,
the score difference at $\bm{\theta}^\star$ equals $m$, giving
\[
A_i(\bm{\theta}^\star)=\mathbf{1}\{m\geq 0\}.
\]

Assume $v>0$. If $m\geq 0$, a rank change occurs when
$\Delta_{\bm{\theta}}<0$, with probability $\Phi(-m/\sqrt{v})$.
If $m<0$, it occurs when $\Delta_{\bm{\theta}}\geq 0$, with
probability $1-\Phi(-m/\sqrt{v})=\Phi(m/\sqrt{v})$.
Both cases therefore give
\[
\mathbb{P}\left[
A_i(\bm{\theta})\neq A_i(\bm{\theta}^\star)
\right]
=
\Phi\left(
-\frac{|\mu_{ij}-\mu_{ik}|}
{\sqrt{V_{ij}+V_{ik}-2C_{ij,ik}}}
\right).
\]
If $v=0$, the score difference equals $m$ almost surely,
so the rank-change probability is zero.
\end{proof}

\begin{figure}
    \centering
    \includegraphics[width=0.95\linewidth]{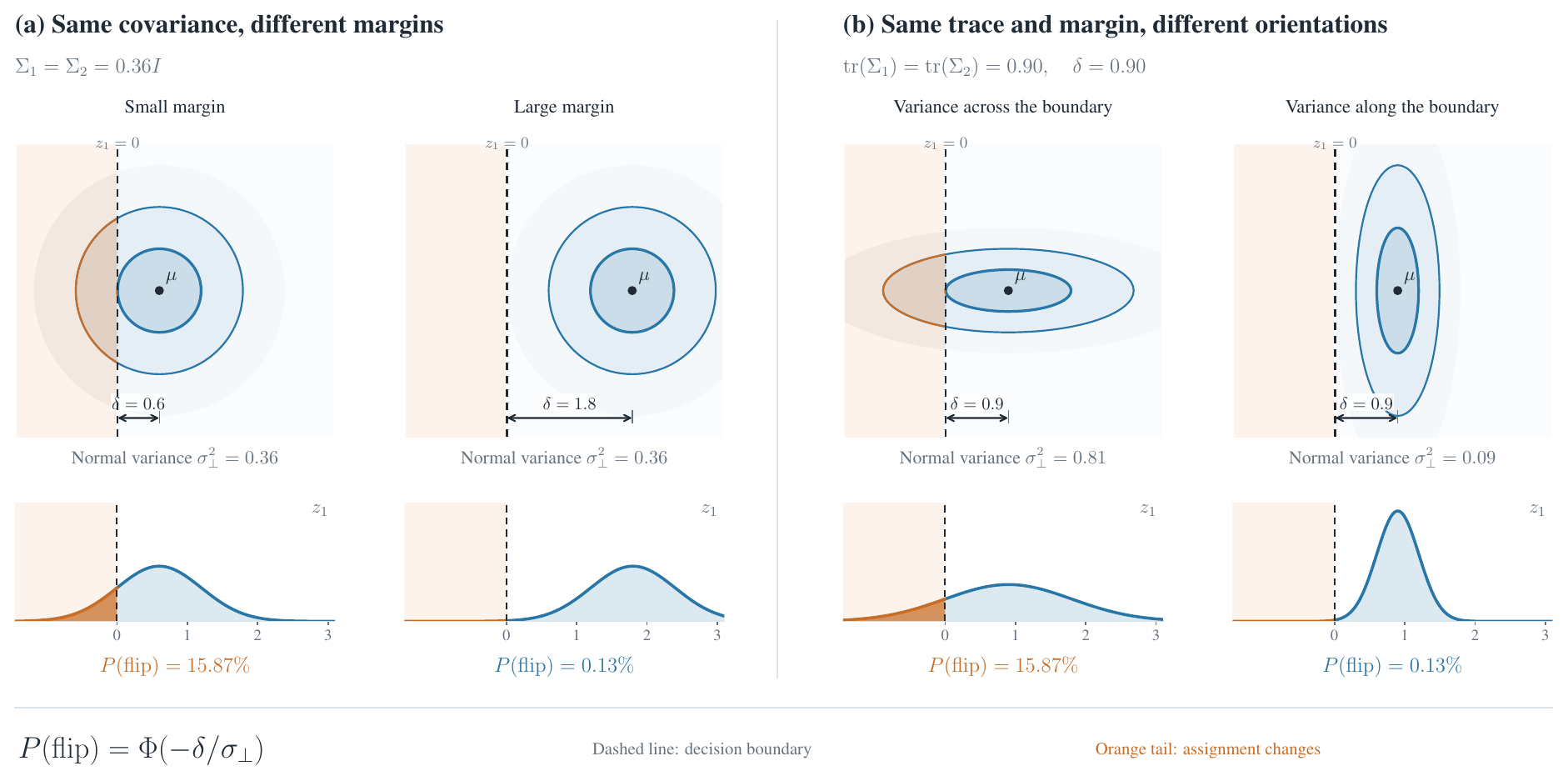}
    \caption{Marginal embedding variance does not determine assignment uncertainty. (a) With identical covariance matrices, a larger margin between the posterior mean and the decision boundary reduces the probability of changing assignment. (b) With identical means and covariance traces, variance perpendicular to the boundary produces greater assignment uncertainty than variance parallel to it. Upper panels show Gaussian embedding posteriors; lower panels show their projections onto the boundary-normal direction. Orange tails indicate assignment changes relative to the posterior mean.}
    \label{fig:assignment_uncertainty}
\end{figure}

\begin{figure}
    \centering
    \includegraphics[width=0.95\linewidth]{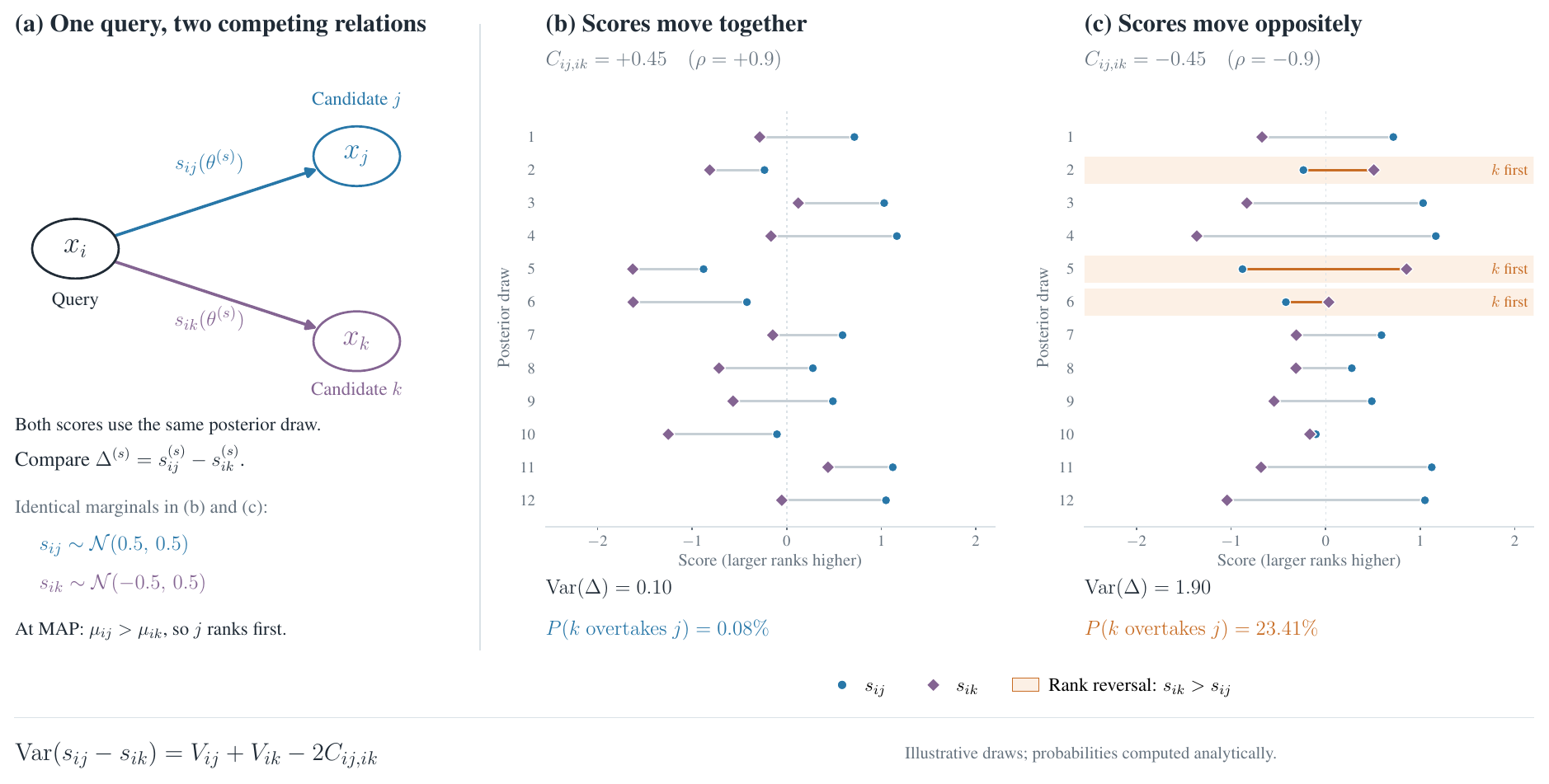}
    \caption{Ranking uncertainty depends on the joint behavior of competing pairwise scores. (a) A query \(x_i\) is compared with candidates \(x_j\) and \(x_k\), with both scores evaluated under the same posterior encoder draw. (b,c) The scores have identical marginal distributions, but positive covariance stabilizes their difference, whereas negative covariance increases ranking uncertainty. Each row connects the two scores from one illustrative draw; orange rows indicate reversals of the MAP ranking. Reported rank-flip probabilities are computed analytically from the joint Gaussian distribution.
}
    \label{fig:score_rank_uncertainty}
\end{figure}

\section{Experimental and Reproducibility Details}
\label{app:experiments}

All experiments use the same posterior-propagation protocol.  We first obtain a MAP contrastive encoder  by loading a pretrained
foundation/self-supervised encoder.  We then fit a last-layer linearized Laplace
approximation (LLA) around the MAP point, sample the last-layer parameters, and apply
the downstream rule separately in each sampled embedding space.  For clustering, the sampled predictions are combined through posterior coassignment probabilities; for classification, the sampled embeddings define either a Monte Carlo predictive readout or a Gaussian-logit readout.  Reliable cores are selected only after this posterior propagation step.

Unless otherwise stated, experiments are repeated over five seeds, and reported aggregate values are means over seeds.  The synthetic two-blob presented in the main text (see Figure~\ref{fig:synthetic:toy})  is an illustrative single-seed experiment.  Ablations use the same architecture and
posterior settings as the corresponding main experiment.

\subsection{Datasets}
\label{app:datasets}

\paragraph{EMNIST Balanced and CIFAR-100.}
  We use EMNIST Balanced \citep{Cohen2017EMNIST}, comprising 47 balanced
  character classes, and CIFAR-100 \citep{krizhevsky2009learning}, comprising 100
  object classes, for classification and curvature-repair ablations. EMNIST
  images are converted to the standard upright orientation and their grayscale
  channel is repeated three times. Images from both datasets are scaled to
  \([0,1]\), resized to \(224\times224\), and normalized using the standard CLIP
  channel-wise statistics. Unless otherwise stated, we use the complete training
  and test splits of both datasets.

\paragraph{DNABERT-S metagenomic panels.}
Genome experiments use the plant and marine panels from the DNABERT-S metagenomic
benchmark \citep{Zhou2024DNABERTS}.  We follow the SemiBin2 binning formulation
\citep{pan2023semibin2}: long contigs define must-link split pairs, and cannot-link pairs
are sampled between contigs.  The calibration split is the corresponding
\(\text{clustering}_0\) panel, while \(\text{binning}_5\) and \(\text{binning}_6\) are
held-out transfer panels. 
The input representation is the 136-dimensional SemiBin2 k-mer feature vector; abundance
and BAM coverage features are not used in this k-mer-only setting.

\paragraph{SICAPv2 pathology.}
SICAPv2 is a public prostate histopathology dataset derived from 155 biopsy whole-slide images from 95 patients \citep{SILVARODRIGUEZ2020105637}. The slides are provided as 18,783 RGB tissue patches of $512\times512$ pixels at $10\times$ magnification, extracted with 50\% overlap; patches containing at most 20\% tissue are excluded. Expert annotations identify non-cancerous tissue and Gleason patterns 3, 4, and 5 at the pixel level. For patch classification, cancerous patches are assigned the majority Gleason pattern in their annotation mask, while non-cancerous patches are drawn from slides classified as non-cancerous, yielding the four labels $\{\mathrm{NC},\mathrm{G3},\mathrm{G4},\mathrm{G5}\}$. We use the official patient-separated test partition, comprising 9,959 training and 2,122 test patches, so patches from the same patient do not occur in both sets.

\subsection{Encoders and Last-Layer Targets}
\label{app:architectures}

Table~\ref{tab:encoder-architectures} summarizes the architectures.  In all cases, the LLA posterior is placed on a final linear map. For encoders taken from prior work, we use the published architecture and weights.

  \begin{table}[t]
  \centering
  \small
  \setlength{\tabcolsep}{4pt}
  \caption{Frozen encoder architectures and linear layers targeted by LLA.
  No feature extractor is trained or fine-tuned in any reported experiment.}
  \label{tab:encoder-architectures}
  \begin{tabular}{p{0.17\linewidth}p{0.34\linewidth}p{0.22\linewidth}p{0.18\linewidth}}
  \toprule
  Setting & Frozen encoder & LLA target & Embedding used downstream \\
  \midrule
  CLIP image experiments and ablations &
  Pretrained CLIP ViT-B/32 visual tower \citep{radford2021learning}, including its
  learned visual projection. Post-projection normalization is not applied before
  the LLA target. &
  CLIP visual projection, a linear map from hidden width 768 to projection
  dimension 512. &
  512-dimensional CLIP visual-projection embedding. \\
  \addlinespace
  SimCLR image experiments &
  Pretrained ImageNet SimCLRv2 ResNet-50 \(1\times\), SK0 checkpoint
  \citep{chen2020big}. Inputs are resized to \(224\times224\), and the converted
  ResNet projection output is retained as the LLA target. &
  Final ResNet \(fc\) layer. &
  Output of the final linear layer. \\
  \addlinespace
  SemiBin2 genome experiments &
  Pretrained SemiBin2 global-environment encoder \citep{pan2023semibin2},
  evaluated on 136-dimensional k-mer features. &
  Last linear layer of the SemiBin2 encoder. &
  SemiBin2 contig embedding. \\
  \addlinespace
  SICAPv2/CONCH pathology &
  Pretrained CONCH-v1.5 vision encoder \citep{Lu2024CONCH}. The LLA implementation
  exposes the final attention-pooling output map before the contrastive layer
  normalization. &
  Final linear output map of the CONCH attention pool. &
  CONCH patch embedding before the final normalization used by the adapter. \\
  \bottomrule
  \end{tabular}
  \end{table}

\subsection{Training and Posterior Fitting}
\label{app:training-posterior}

\paragraph{Frozen-CLIP loss-and-repair ablation.}
We study curvature repair on CIFAR-100 using a frozen pretrained CLIP
ViT-B/32 visual encoder, without post-projection embedding normalization. KFAC GGN curvature is fitted on the final linear projection using either the distance-Bernoulli loss or the contrastive margin loss, each combined with spectral or radial-only repair. We accumulate at most 64 curvature batches of 16 supervised pairs, sampled with equal positive and negative
proportions. Radial-only repair discards the tangential component, with
$\epsilon_{\mathrm{tan}}=0$. All four configurations are evaluated over
five random seeds.

We reserve 50 images per class, excluding them from curvature fitting and classifier training in all four configurations. For the margin loss, $m$ is set to the 25th percentile of negative-pair distances on this held-out subset. The same validation pairs and selected margin are used for both repair modes; the margin is not selected by classification performance.

For posterior-predictive classification, the prior precision is selected from \(\{1,10^2,10^4\}\) jointly with the logistic-classifier regularization parameter \(C\in\{1,10^3,10^4\}\). Selection uses a second, disjoint balanced validation subset of 50 images per class drawn from the remaining training images. For each value of \(C\), a logistic classifier is fitted to the MAP embeddings of the corresponding training partition and held fixed while each candidate prior is evaluated using four posterior samples. Candidate pairs are ranked by validation negative log-likelihood. After selection, we draw 64 posterior samples. For each sample, we compute the corresponding training and test embeddings, fit a separate logistic classifier using the selected \(C\), and obtain test-set probabilities. Final posterior-predictive probabilities are computed by averaging across the 64 independently fitted classifiers.

\paragraph{CLIP and SimCLR image experiments.}
The pretrained image backbones are kept frozen, and KFAC factors are fitted on
supervised InfoNCE candidate-set batches from the training split. The candidate
set size is 10, the curvature batch size is 64, and at most 64 curvature
batches are accumulated. We use temperature \(0.07\) for CLIP and \(0.1\)
for SimCLR, and draw 16 posterior samples for the final evaluation. For
classification, the prior precision is selected from
\(\{1,10^2,10^4\}\), jointly with
\(C\in\{1,10^3,10^4\}\), using a balanced training-only validation subset.
Each candidate is evaluated using four posterior samples and selected
primarily by validation negative log-likelihood. For clustering, the prior
precision is selected from
\(\{10^3,10^4,10^5,10^6\}\) using the number of high-quality clusters
obtained by posterior-consensus clustering on the training split, with
macro-F1 and Hungarian-aligned accuracy as tie-breakers. We evaluate both
empirical-Fisher KFAC and GGN KFAC. Neither configuration requires a curvature
repair because the InfoNCE output Hessian is positive semidefinite.

\paragraph{SemiBin2 genome experiments.}
The pretrained SemiBin2 encoder is kept fixed, and KFAC is fitted on its final
encoder linear layer. Each KFAC fit uses 64 curvature batches, a positive batch
size of 512, one randomly sampled cannot-link pair per positive pair, and
margin \(m=1\). The prior-precision grid is
\(\{10^3,10^4,10^5,10^6\}\). Prior selection is performed separately for
each source dataset, random seed, and KFAC configuration on the corresponding
\texttt{clustering\_0} calibration panel. The selected posterior is then
transferred unchanged to \texttt{binning\_5} and \texttt{binning\_6}.
The experiments use 20 K-means initializations, and draw
32 posterior samples. Because the SemiBin2 objective is a vector-difference
margin loss, the GGN configurations use either spectral repair or radial-only
repair; empirical-Fisher KFAC with no repair is included as a positive
semidefinite alternative. Also, following Metabat and DNABERT \citep{Zhou2024DNABERTS, kang2015metabat}, we only keep sequences longer than 2500 bp and we filter out species with fewer than 10 sequences.

\paragraph{SICAPv2 pathology classification.}
The CLIP and CONCH-v1.5 pathology backbones are kept frozen, and KFAC is fitted
on the corresponding final linear projection. CLIP is evaluated at
\(224\times224\) resolution using CLIP normalization, whereas CONCH-v1.5 is
evaluated at \(448\times448\) resolution using ImageNet normalization. Both
models use supervised InfoNCE candidate sets of size 8 and accumulate at most
256 curvature batches. The curvature batch size is 16 for CLIP and 4 for
CONCH; the corresponding temperatures are \(0.07\) and \(0.1\). The prior
precision is selected from \(\{1,10^2,10^4\}\), jointly with
\(C\in\{1,10^3,10^4\}\), on a balanced training-only validation subset using
four posterior samples per candidate. Final classification results use 16
posterior samples. We evaluate empirical-Fisher KFAC and GGN KFAC without
repair, with damping \(10^{-4}\) and jitter \(10^{-6}\).    

\subsection{Downstream Rules and Metrics}
\label{app:downstream-metrics}

\paragraph{Clustering.}
MAP clustering applies K-means in the MAP embedding space with \(K=500\) except for the K-ablation, where we vary this value.  Posterior clustering draws embeddings from the LLA posterior, reruns the same K-means rule for each
draw, and forms a posterior-consensus clustering from the coassignment matrix.  

Cluster labels are evaluated by Hungarian alignment to ground-truth labels.  We report mainly the number of High-Quality clusters in the experimental section, and also macro-F1 and ARI in the additional Genome Curvature comparison (see Figure~\ref{fig:genome-curvature-differences-k-500}).  A high-quality cluster is one whose aligned retained-set F1 exceeds the stated threshold, usually 0.9.  Figure-style stacked bars count clusters with F1 greater than 0.5, 0.6, 0.7, 0.8, and 0.9.

\paragraph{Classification.}
For image classification, the MAP readout is multinomial logistic regression trained on
MAP embeddings.  The sampled-LLA readout trains the same logistic classifier in each
sampled embedding space and averages predictive probabilities.  The Gaussian-logit readout
trains a single downstream logistic model against the Gaussian embedding distribution
induced by the last-layer posterior, then averages predictive probabilities under that
distribution.  Classification metrics are top-1 accuracy, negative log likelihood, and 15-bin top-label expected calibration error.

\paragraph{Reliable cores.}
For clustering, posterior reliability for an item is the average posterior agreement between
that item and the other items it has been clustered with.  Reliable-core curves retain either
the globally highest-reliability items or the highest-reliability items within each predicted
cluster.  The main coverage curves use the per-cluster rule.  Baselines include pointwise
embedding-trace rejection, MAP centroid similarity for genome coverage curves, and repeated
random rejection. 

\section{Additional Clustering Experimental Results}
\label{sec:additional-clustering-results}

\begin{table}[t]
  \centering
  \caption{Numbers of clusters exceeding each aligned F1 threshold for direct SemiBin2 MAP assignments and posterior-consensus solutions, using $K=500$ from the cluster-count ablation. Values are mean $\pm$ standard deviation over five seeds.}
  \label{tab:genome-f1-thresholds-k-500}
  \resizebox{\textwidth}{!}{%
  \begin{tabular}{lrrrrr}
    \toprule
    Method & F1 $> 0.9$ & F1 $> 0.8$ & F1 $> 0.7$ & F1 $> 0.6$ & F1 $> 0.5$ \\
    \midrule
    \multicolumn{6}{l}{\textbf{Marine 5}} \\
    MAP & $20.4 \pm 1.1$ & $56.8 \pm 2.2$ & $97.6 \pm 3.4$ & $150.8 \pm 3.4$ & $204.8 \pm 5.3$ \\
    Consensus, EF & $42.0 \pm 2.8$ & $98.4 \pm 2.9$ & $158.8 \pm 1.9$ & $222.2 \pm 1.9$ & $272.0 \pm 0.7$ \\
    Consensus, GGN/radial & $42.8 \pm 1.5$ & $97.0 \pm 1.4$ & $156.4 \pm 1.9$ & $217.2 \pm 2.7$ & $268.2 \pm 3.2$ \\
    Consensus, GGN/spectral & $42.6 \pm 2.1$ & $97.8 \pm 1.9$ & $159.0 \pm 0.7$ & $220.8 \pm 3.7$ & $268.2 \pm 2.3$ \\
    \midrule
    \multicolumn{6}{l}{\textbf{Marine 6}} \\
    MAP & $18.6 \pm 0.9$ & $59.4 \pm 1.5$ & $114.2 \pm 5.1$ & $161.2 \pm 6.0$ & $204.0 \pm 7.0$ \\
    Consensus, EF & $49.2 \pm 1.3$ & $104.2 \pm 2.4$ & $175.0 \pm 1.6$ & $230.4 \pm 2.7$ & $268.2 \pm 2.5$ \\
    Consensus, GGN/radial & $50.4 \pm 3.5$ & $106.0 \pm 4.8$ & $176.6 \pm 3.0$ & $229.8 \pm 2.3$ & $269.8 \pm 3.0$ \\
    Consensus, GGN/spectral & $47.4 \pm 1.7$ & $104.2 \pm 2.9$ & $174.4 \pm 3.5$ & $230.0 \pm 2.5$ & $267.8 \pm 1.9$ \\
    \midrule
    \multicolumn{6}{l}{\textbf{Plant 5}} \\
    MAP & $4.4 \pm 1.1$ & $10.2 \pm 1.5$ & $16.0 \pm 2.0$ & $24.2 \pm 0.8$ & $36.6 \pm 2.7$ \\
    Consensus, EF & $8.4 \pm 1.1$ & $17.2 \pm 1.3$ & $30.2 \pm 1.3$ & $42.8 \pm 1.5$ & $51.2 \pm 0.8$ \\
    Consensus, GGN/radial & $9.2 \pm 0.8$ & $17.6 \pm 1.1$ & $28.8 \pm 1.1$ & $42.6 \pm 2.6$ & $51.8 \pm 2.5$ \\
    Consensus, GGN/spectral & $8.8 \pm 1.3$ & $18.6 \pm 0.5$ & $31.2 \pm 1.3$ & $42.0 \pm 1.2$ & $52.0 \pm 1.6$ \\
    \midrule
    \multicolumn{6}{l}{\textbf{Plant 6}} \\
    MAP & $4.6 \pm 1.9$ & $10.0 \pm 2.3$ & $15.4 \pm 0.9$ & $24.2 \pm 1.3$ & $36.2 \pm 2.6$ \\
    Consensus, EF & $13.2 \pm 0.4$ & $22.6 \pm 0.9$ & $29.6 \pm 0.9$ & $41.8 \pm 0.4$ & $51.0 \pm 1.2$ \\
    Consensus, GGN/radial & $13.2 \pm 0.4$ & $23.4 \pm 0.5$ & $30.0 \pm 1.2$ & $41.4 \pm 1.1$ & $50.8 \pm 2.0$ \\
    Consensus, GGN/spectral & $13.0 \pm 0.7$ & $23.0 \pm 1.4$ & $29.6 \pm 1.1$ & $40.8 \pm 1.3$ & $50.2 \pm 0.8$ \\
    \bottomrule
  \end{tabular}
  }
\end{table}

\begin{table}[t]
  \centering
  \caption{Numbers of high-quality clusters (aligned F1 $>0.9$) retained at selected coverage levels under consensus and MAP selection rules, using $K=500$ from the cluster-count ablation. Values are mean $\pm$ standard deviation over five seeds.}
  \label{tab:genome-hq-coverage-k-500}
  \resizebox{\textwidth}{!}{%
  \begin{tabular}{lrrrrr}
    \toprule
    Selector & 90\% & 70\% & 50\% & 30\% & 10\% \\
    \midrule
    \multicolumn{6}{l}{\textbf{Marine 5}} \\
    Consensus + rel. (EF) & $64.4 \pm 2.3$ & $97.0 \pm 3.4$ & $118.8 \pm 3.1$ & $132.2 \pm 3.6$ & $153.0 \pm 5.2$ \\
    Consensus + rel. (GGN/radial) & $63.6 \pm 2.3$ & $96.0 \pm 3.1$ & $114.6 \pm 2.3$ & $133.2 \pm 4.4$ & $152.6 \pm 8.6$ \\
    Consensus + rel. (GGN/spectral) & $65.0 \pm 1.4$ & $95.6 \pm 3.1$ & $118.4 \pm 1.9$ & $136.0 \pm 2.3$ & $154.0 \pm 3.7$ \\
    Consensus + random (EF) & $40.8 \pm 2.8$ & $41.5 \pm 2.4$ & $44.3 \pm 3.1$ & $42.0 \pm 2.9$ & $50.7 \pm 2.0$ \\
    Consensus + random (GGN/radial) & $41.1 \pm 0.8$ & $42.1 \pm 1.3$ & $44.3 \pm 1.4$ & $42.5 \pm 0.9$ & $50.8 \pm 1.5$ \\
    Consensus + random (GGN/spectral) & $40.5 \pm 1.4$ & $41.2 \pm 1.7$ & $44.6 \pm 1.3$ & $41.9 \pm 1.6$ & $50.4 \pm 1.4$ \\
    MAP + same reliability & $34.6 \pm 3.8$ & $56.6 \pm 1.5$ & $76.2 \pm 2.2$ & $99.8 \pm 4.9$ & $120.2 \pm 3.1$ \\
    MAP + centroid distance & $30.2 \pm 2.3$ & $49.2 \pm 3.3$ & $66.8 \pm 2.4$ & $84.0 \pm 3.4$ & $99.2 \pm 1.1$ \\
    MAP + margin & $27.6 \pm 2.1$ & $50.8 \pm 2.2$ & $65.0 \pm 2.5$ & $84.6 \pm 2.4$ & $105.4 \pm 8.4$ \\
    MAP + LLA Trace & $20.8 \pm 1.1$ & $20.0 \pm 1.9$ & $19.6 \pm 1.5$ & $18.2 \pm 1.6$ & $16.4 \pm 1.1$ \\
    \midrule
    \multicolumn{6}{l}{\textbf{Marine 6}} \\
    Consensus + rel. (EF) & $73.0 \pm 1.9$ & $106.6 \pm 1.3$ & $133.0 \pm 2.0$ & $150.8 \pm 3.6$ & $165.6 \pm 3.3$ \\
    Consensus + rel. (GGN/radial) & $76.2 \pm 1.6$ & $110.8 \pm 3.3$ & $135.6 \pm 3.0$ & $153.2 \pm 3.2$ & $165.8 \pm 5.7$ \\
    Consensus + rel. (GGN/spectral) & $75.4 \pm 2.3$ & $108.4 \pm 4.0$ & $132.2 \pm 2.6$ & $150.2 \pm 2.2$ & $165.8 \pm 5.8$ \\
    Consensus + random (EF) & $46.5 \pm 1.9$ & $46.0 \pm 1.6$ & $50.0 \pm 1.3$ & $46.0 \pm 1.3$ & $56.9 \pm 1.7$ \\
    Consensus + random (GGN/radial) & $47.5 \pm 2.6$ & $46.9 \pm 2.2$ & $51.0 \pm 2.7$ & $47.6 \pm 1.6$ & $57.6 \pm 2.0$ \\
    Consensus + random (GGN/spectral) & $44.7 \pm 1.8$ & $44.5 \pm 1.9$ & $48.6 \pm 2.4$ & $45.0 \pm 1.8$ & $55.1 \pm 2.1$ \\
    MAP + same reliability & $34.4 \pm 2.6$ & $62.8 \pm 2.4$ & $89.2 \pm 5.4$ & $103.4 \pm 3.4$ & $119.8 \pm 2.7$ \\
    MAP + centroid distance & $32.0 \pm 2.7$ & $54.2 \pm 3.5$ & $72.0 \pm 3.2$ & $92.0 \pm 3.7$ & $106.8 \pm 3.3$ \\
    MAP + margin & $30.2 \pm 3.0$ & $53.2 \pm 1.3$ & $76.0 \pm 4.4$ & $96.2 \pm 2.7$ & $112.6 \pm 3.2$ \\
    MAP + LLA Trace & $19.8 \pm 1.1$ & $18.2 \pm 0.8$ & $16.8 \pm 1.1$ & $15.8 \pm 1.5$ & $14.4 \pm 3.8$ \\
    \midrule
    \multicolumn{6}{l}{\textbf{Plant 5}} \\
    Consensus + rel. (EF) & $14.6 \pm 0.9$ & $18.4 \pm 0.9$ & $20.4 \pm 0.9$ & $22.4 \pm 1.3$ & $24.6 \pm 2.9$ \\
    Consensus + rel. (GGN/radial) & $14.2 \pm 1.3$ & $18.2 \pm 0.8$ & $21.0 \pm 1.7$ & $22.6 \pm 0.9$ & $24.2 \pm 1.3$ \\
    Consensus + rel. (GGN/spectral) & $15.4 \pm 1.7$ & $19.0 \pm 1.0$ & $21.6 \pm 1.5$ & $22.8 \pm 1.8$ & $26.2 \pm 1.3$ \\
    Consensus + random (EF) & $7.7 \pm 0.7$ & $8.1 \pm 1.1$ & $10.4 \pm 1.0$ & $8.6 \pm 1.2$ & $12.6 \pm 1.0$ \\
    Consensus + random (GGN/radial) & $7.9 \pm 0.8$ & $8.3 \pm 0.8$ & $11.0 \pm 0.9$ & $7.8 \pm 0.5$ & $13.1 \pm 0.7$ \\
    Consensus + random (GGN/spectral) & $8.5 \pm 1.1$ & $8.8 \pm 1.0$ & $11.5 \pm 1.0$ & $8.4 \pm 0.8$ & $13.0 \pm 0.8$ \\
    MAP + same reliability & $8.8 \pm 1.8$ & $9.4 \pm 1.5$ & $11.4 \pm 1.7$ & $15.8 \pm 1.6$ & $21.8 \pm 2.8$ \\
    MAP + centroid distance & $9.0 \pm 1.6$ & $9.8 \pm 2.2$ & $11.6 \pm 0.9$ & $14.4 \pm 0.9$ & $16.4 \pm 1.1$ \\
    MAP + margin & $6.0 \pm 1.2$ & $8.8 \pm 1.9$ & $10.0 \pm 1.4$ & $13.6 \pm 2.5$ & $19.6 \pm 2.4$ \\
    MAP + LLA Trace & $4.8 \pm 1.5$ & $6.0 \pm 1.4$ & $5.6 \pm 1.5$ & $5.8 \pm 1.5$ & $5.8 \pm 1.8$ \\
    \midrule
    \multicolumn{6}{l}{\textbf{Plant 6}} \\
    Consensus + rel. (EF) & $17.0 \pm 1.4$ & $23.8 \pm 1.1$ & $27.0 \pm 1.2$ & $30.6 \pm 1.8$ & $34.4 \pm 1.1$ \\
    Consensus + rel. (GGN/radial) & $17.2 \pm 0.8$ & $25.6 \pm 0.9$ & $27.6 \pm 0.9$ & $32.0 \pm 1.6$ & $35.0 \pm 1.9$ \\
    Consensus + rel. (GGN/spectral) & $17.4 \pm 1.5$ & $25.6 \pm 0.5$ & $27.6 \pm 0.9$ & $31.8 \pm 1.8$ & $35.4 \pm 1.1$ \\
    Consensus + random (EF) & $10.4 \pm 0.7$ & $10.8 \pm 0.9$ & $13.8 \pm 0.9$ & $10.5 \pm 0.5$ & $16.1 \pm 0.4$ \\
    Consensus + random (GGN/radial) & $10.6 \pm 1.1$ & $11.2 \pm 1.0$ & $14.3 \pm 0.3$ & $11.0 \pm 0.4$ & $16.2 \pm 0.3$ \\
    Consensus + random (GGN/spectral) & $10.1 \pm 1.4$ & $10.5 \pm 1.1$ & $14.1 \pm 0.6$ & $11.1 \pm 0.5$ & $16.2 \pm 0.5$ \\
    MAP + same reliability & $6.0 \pm 1.9$ & $8.4 \pm 2.8$ & $12.6 \pm 2.1$ & $18.2 \pm 1.6$ & $25.8 \pm 2.3$ \\
    MAP + centroid distance & $6.8 \pm 1.8$ & $11.0 \pm 2.0$ & $12.8 \pm 1.9$ & $15.8 \pm 2.8$ & $18.2 \pm 3.6$ \\
    MAP + margin & $6.0 \pm 1.9$ & $8.2 \pm 2.2$ & $11.8 \pm 2.3$ & $17.0 \pm 2.0$ & $22.4 \pm 2.3$ \\
    MAP + LLA Trace & $5.2 \pm 2.2$ & $5.6 \pm 2.1$ & $6.4 \pm 2.5$ & $4.8 \pm 1.6$ & $5.2 \pm 1.8$ \\
    \bottomrule
  \end{tabular}
  }
\end{table}

\begin{figure}[t]
  \centering
  \includegraphics[width=\textwidth]{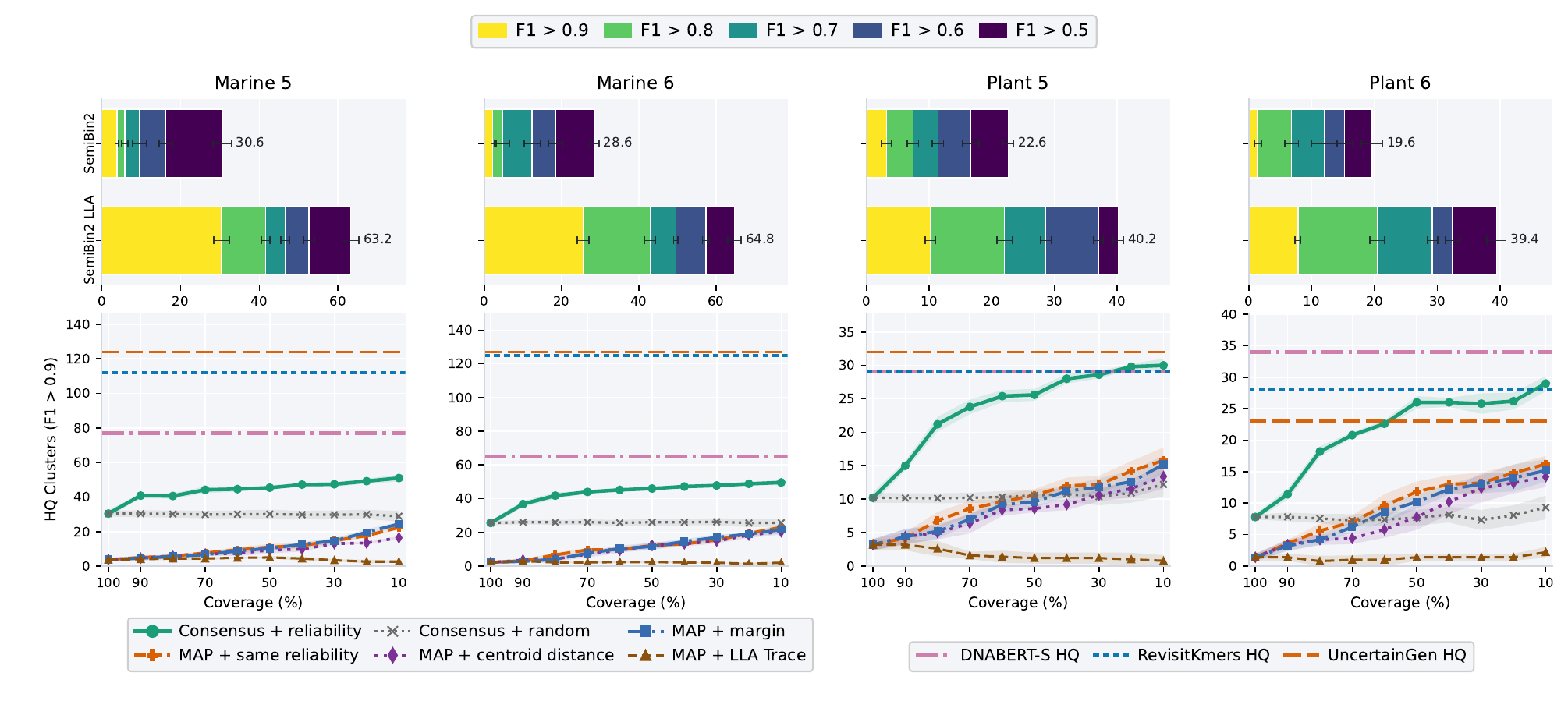}
  \caption{Genome reliable-core results on Marine and Plant splits, using $K=100$ in k-Means, and using the GGN approximation with spectral repair. Top panels compare high-quality cluster counts for SemiBin2 and the LLA posterior-consensus variant across F1 thresholds. Bottom panels show high-quality clusters retained as coverage is reduced with different approaches. }
  \label{fig:genome:k:100}
\end{figure}

\begin{figure}[t]
  \centering
  \includegraphics[width=\textwidth]{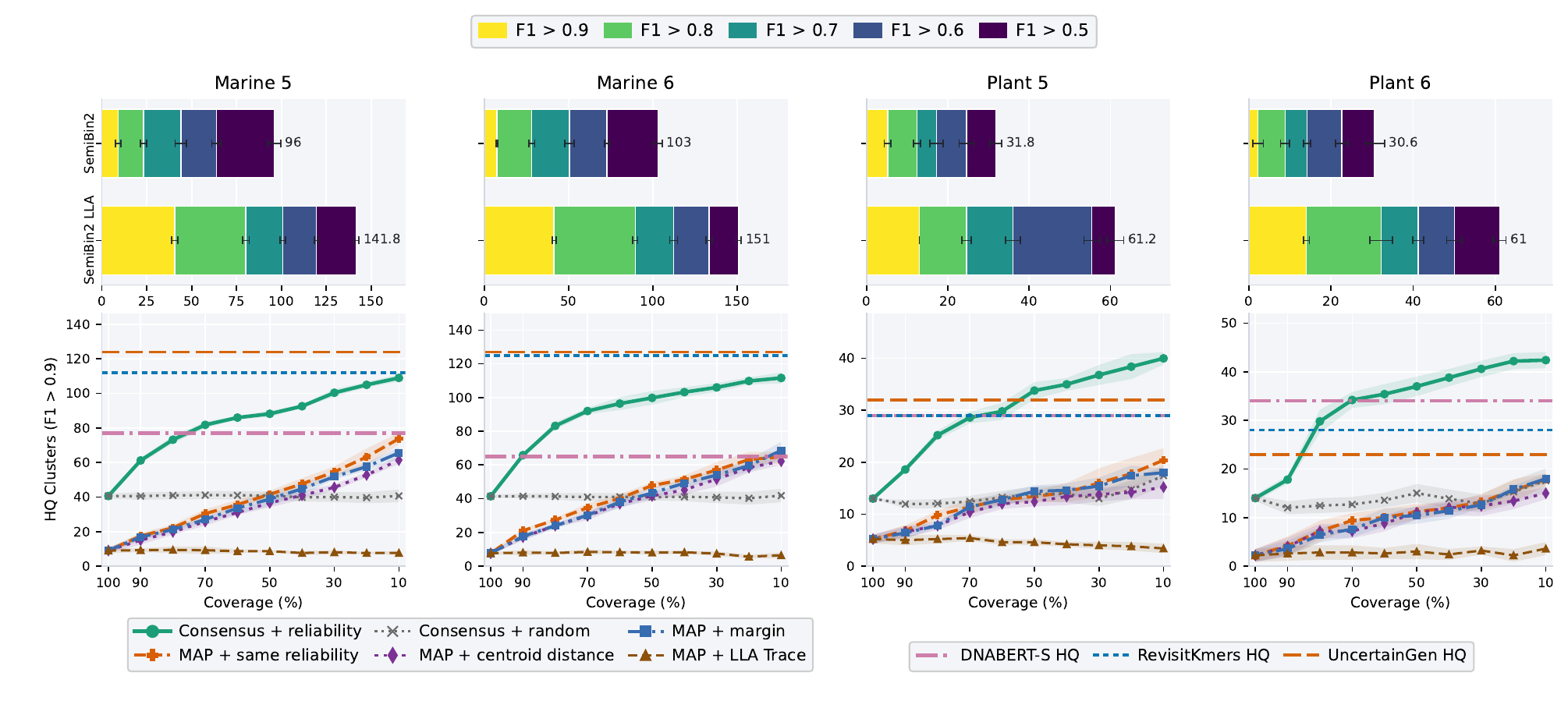}
  \caption{Genome reliable-core results on Marine and Plant splits, using $K=200$ in k-Means, and using the GGN approximation with spectral repair. Top panels compare high-quality cluster counts for SemiBin2 and the LLA posterior-consensus variant across F1 thresholds. Bottom panels show high-quality clusters retained as coverage is reduced with different approaches. }
  \label{fig:genome:k:200}
\end{figure}

\begin{figure}[t]
  \centering
  \includegraphics[width=\textwidth]{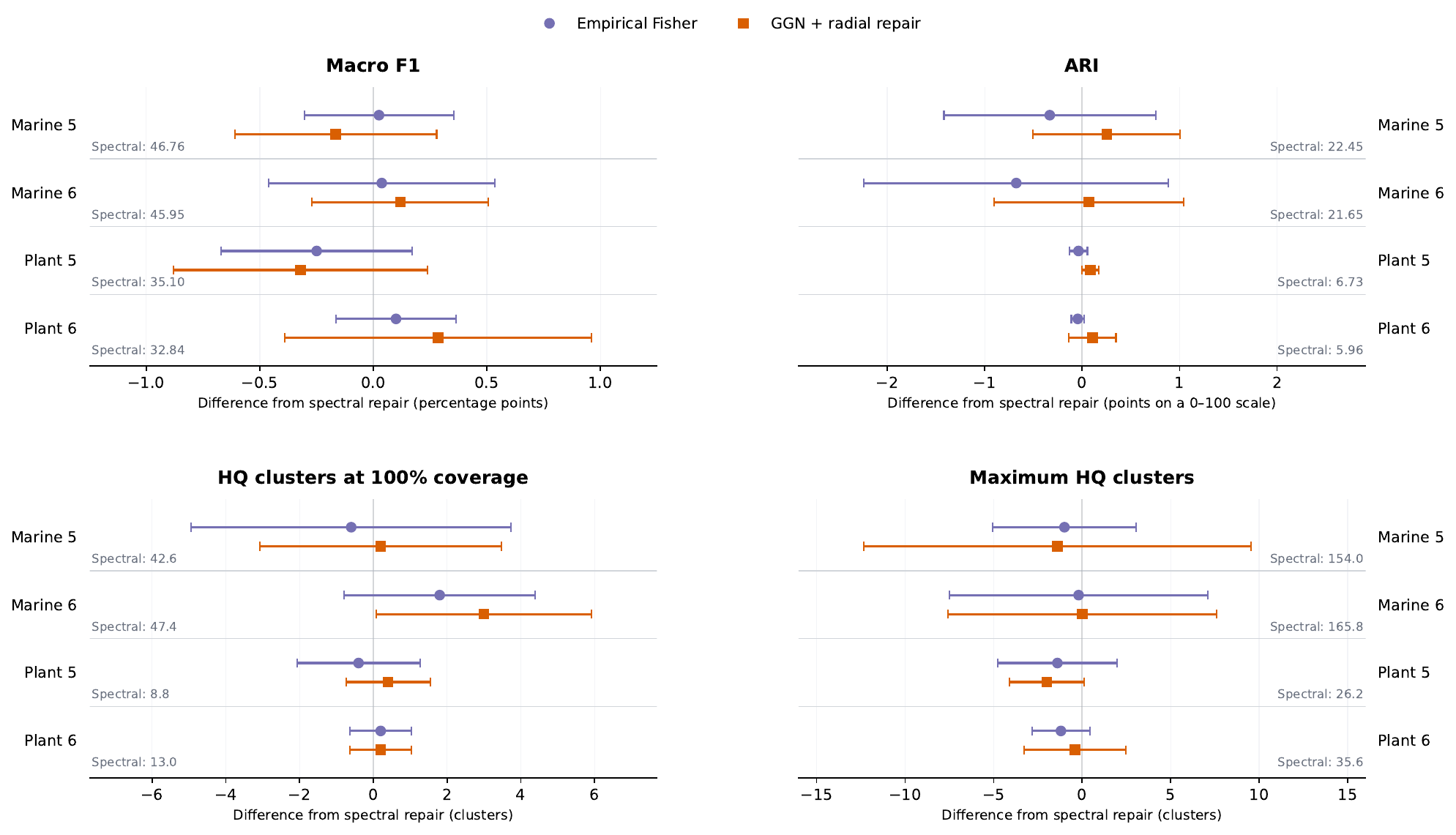}
  \caption{Genome curvature comparison at $K=500$, using GGN with spectral repair as the reference. Points show mean paired differences (alternative minus spectral) across five seeds; error bars indicate one standard deviation of those differences. Macro F1 and ARI differences are expressed on a 0--100 scale, and HQ differences are expressed in cluster counts. Maximum HQ counts are computed separately for each seed over the 10--100\% coverage grid using consensus-based reliability selection before taking paired differences.}
  \label{fig:genome-curvature-differences-k-500}
\end{figure}

We provide additional metagenomic binning clustering results performed on the main text of the paper. Tables \ref{tab:genome-f1-thresholds-k-500} and \ref{tab:genome-hq-coverage-k-500} provide numerical results to Figure~\ref{fig:genome-reliable-cores}.  Figures \ref{fig:genome:k:100} and \ref{fig:genome:k:200} show the results of the task when using $K=100$ and $K=200$, respectively.

Figure~\ref{fig:genome-curvature-differences-k-500} isolates the differences between the curvature configurations and repairs. The figure shows the difference between the results obtained by  GGN with spectral repair, with EF and GGN with radial repair, respectively. Notice that most mean differences lie close to zero: absolute Macro F1 differences remain below 0.33 percentage points, ARI differences remain below 0.69 points on the 0--100 scale, and full-coverage HQ differences are at most three clusters. Mean maximum HQ counts differ by at most two clusters. These results indicate broadly comparable clustering outcomes across the configurations, with the largest differences occurring in full-coverage HQ counts on Marine~6.

\section{Ablation Studies}
\label{app:ablations}

\paragraph{Number of posterior samples.}
The sample-count ablation varies
\[
S\in\{1,2,4,8,16,32,64\}.
\]
For images, this grid is run over both datasets, both pretrained backbones, see the results in Figure~\ref{fig:image:posterior:sample:ablation}.

Increasing the number of posterior samples primarily improves predictive uncertainty rather than classification accuracy. On CIFAR-100, NLL and ECE improve consistently as more samples are averaged, with most of the gain obtained by approximately $16$ or $32$ samples and only modest changes thereafter. Accuracy also improves slightly once multiple samples are used. On EMNIST, the MAP classifier
is already strong and posterior averaging produces comparatively small
changes, but NLL and ECE generally remain improved. The empirical-Fisher and GGN results exhibit similar convergence patterns.

For genome binning, the same grid is run for the SemiBin2 Plant/Marine
fit-transfer protocol. In this case, since the margin loss is used, we present the results using spectral and radial repairs, and also empirical fisher.  The results of this experiments are presented in Figure~\ref{fig:clustering:posterior:sample:ablation}. In this experiment, we also observe an improvement in the metrics when we increase the number of posterior samples, stabilizing in most of the cases when using $8$ samples.

\paragraph{Number of clusters in genome K-means.}
The K-ablation is restricted to SemiBin2 genome K-means.  Instead of using oracle \(K\),
we evaluate
\[
K\in\{100,200,500\}.
\]
The posterior sample count remains 32.  This ablation tests whether posterior reliability and high-quality-cluster counts remain stable when the downstream clustering rule is moved away from the true number of reference labels. The results, displayed in Figure~\ref{fig:semibin2-k-ablation},  show that while in the Marine splits increasing the number of clusters $K$ leads to an improvement in the number of HQ clusters, they do not improve in the Plant splits when increasing $K$ from $200$ to $500$.

\begin{figure}
    \centering
    \includegraphics[width=\linewidth]{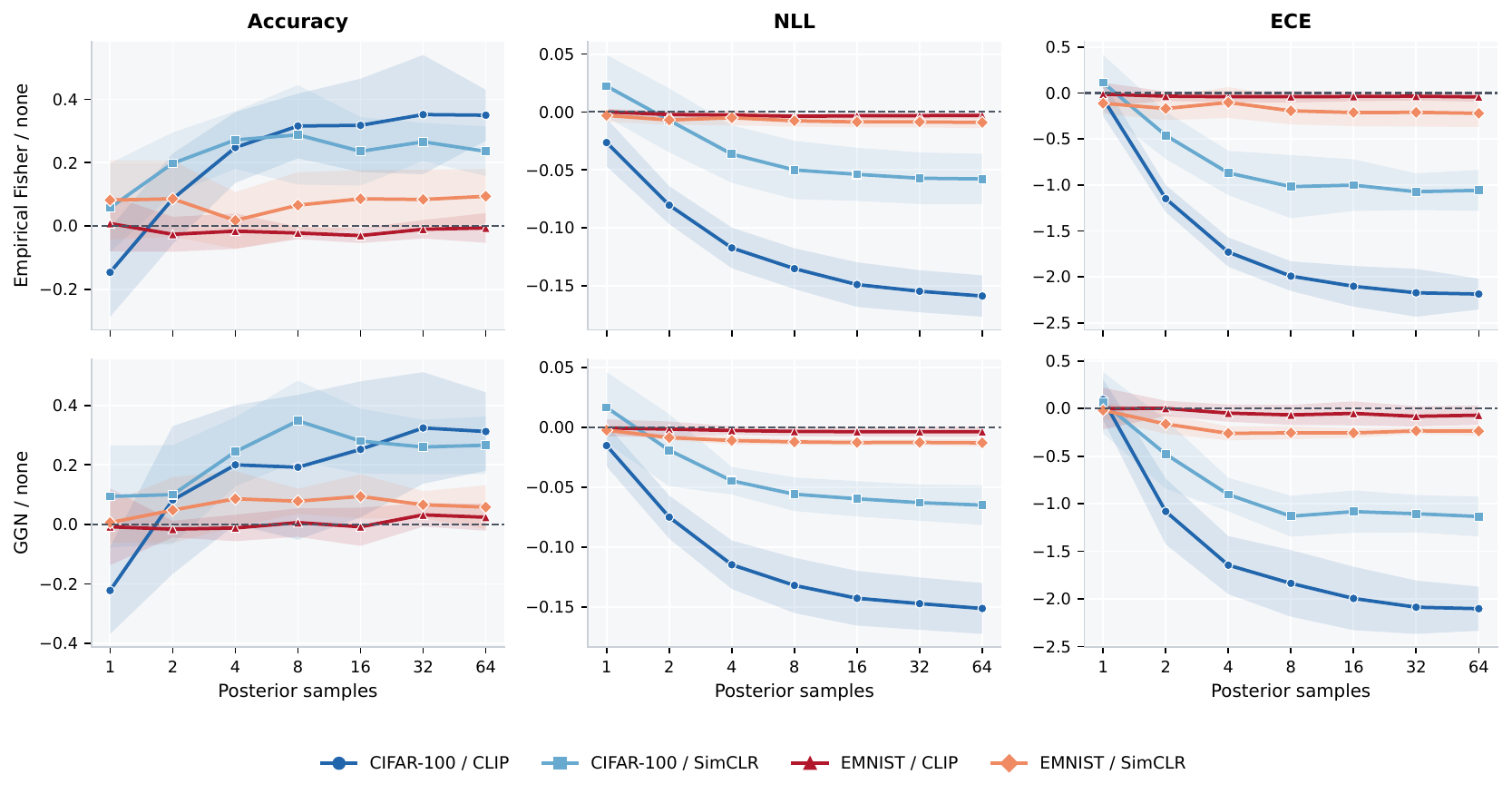}
    \caption{Image classification posterior-sample ablation. We report the results as the difference between the performance of our proposal and that of the MAP solution. No repair step is required in these experiments, as the InfoNCE loss is used.}
    \label{fig:image:posterior:sample:ablation}
\end{figure}

\begin{figure}
    \centering
    \includegraphics[width=\linewidth]{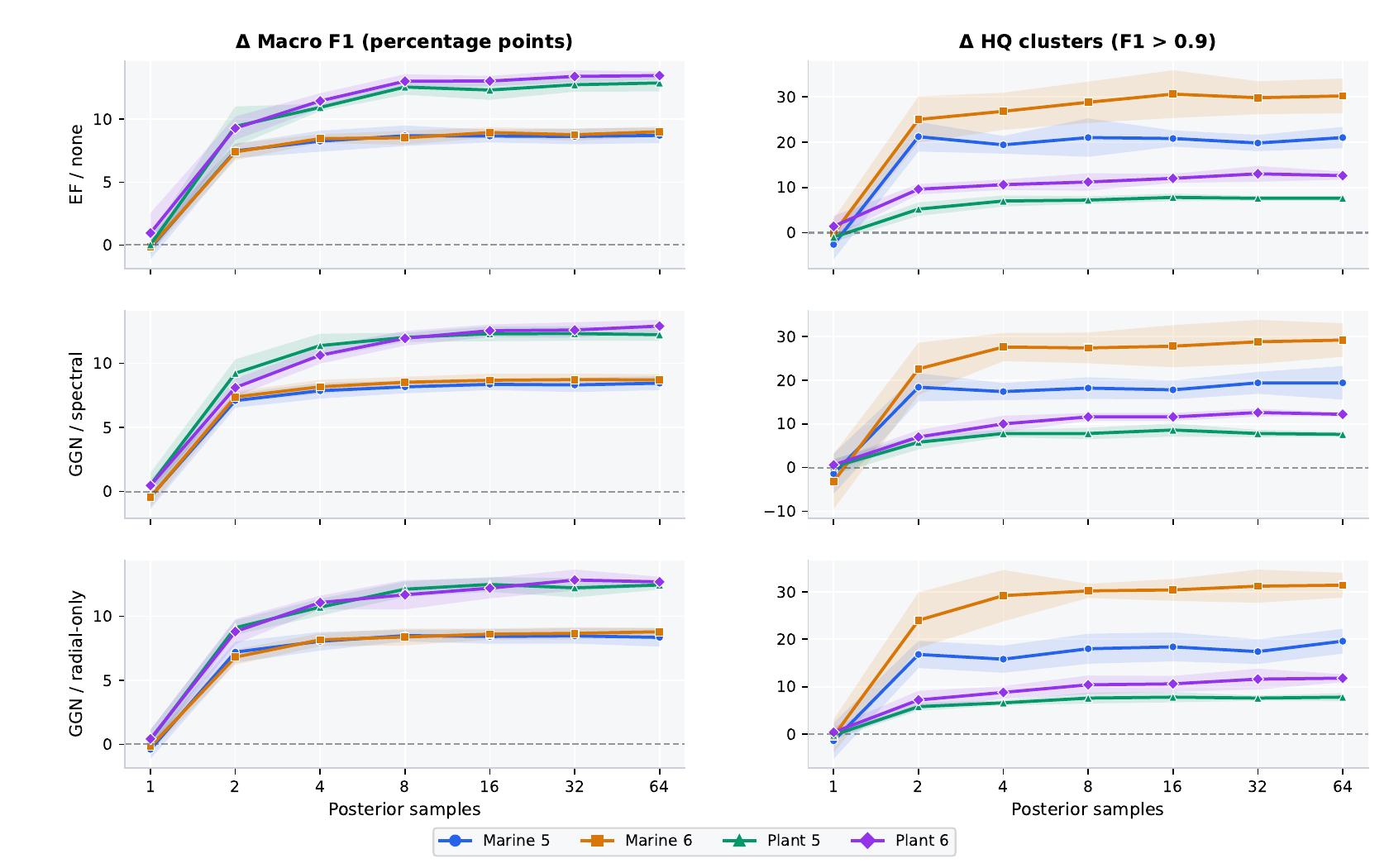}
    \caption{SemiBin2 posterior-sample ablation. We report the results as the difference between the performance of our proposal and that of the MAP solution.}
    \label{fig:clustering:posterior:sample:ablation}
\end{figure}

\begin{figure}
    \centering
    \includegraphics[width=\linewidth]{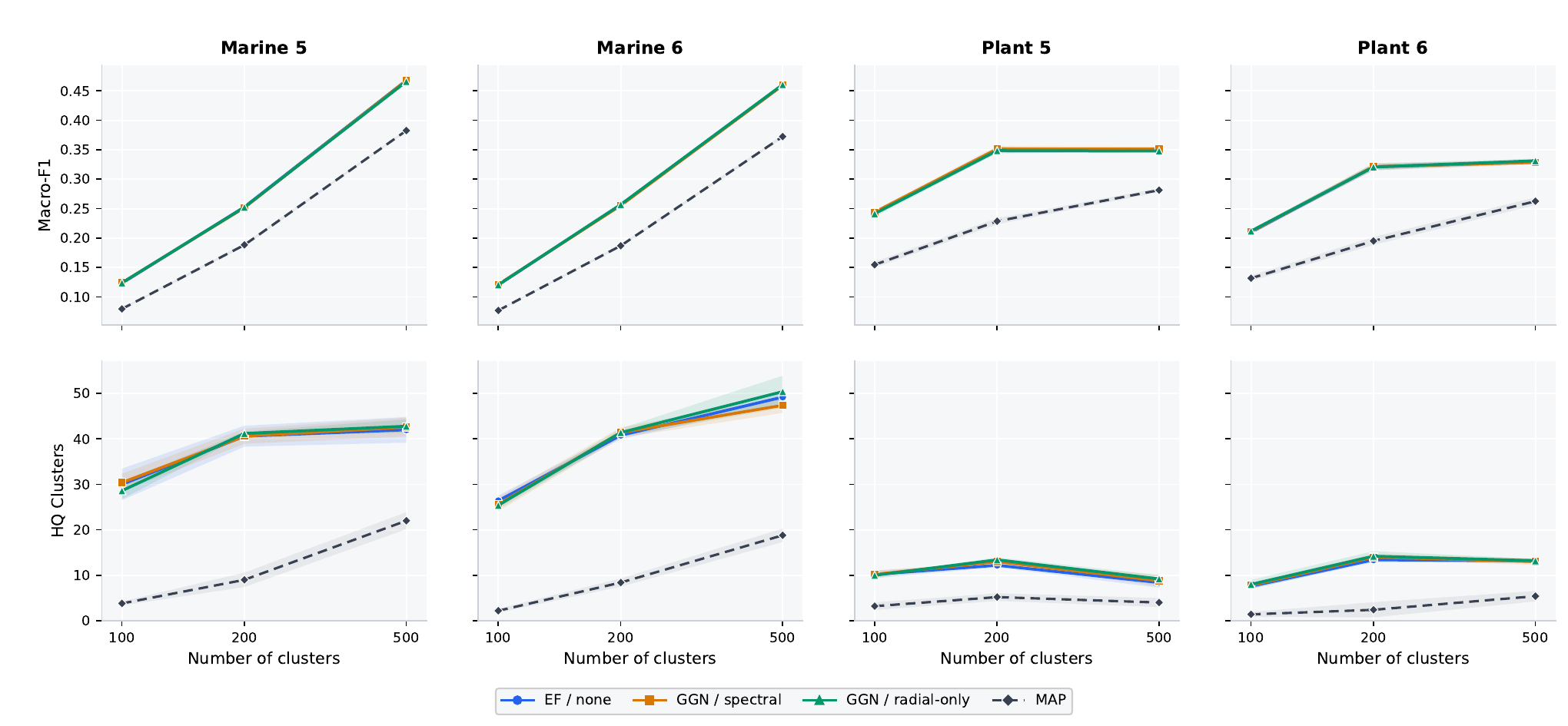}
    \caption{SemiBin2 clustering-$K$ ablation. Values are mean $\pm$ one standard deviation over five seeds at full coverage. MAP denotes standard SemiBin2 clustering and is shown once per $K$ using the EF/none run outputs.}
    \label{fig:semibin2-k-ablation}
\end{figure}

\end{document}